\documentclass[a4paper,11pt]{article}

\usepackage{amsmath}
\usepackage{amssymb}
\usepackage{amsthm}
\usepackage{authblk}
\usepackage{graphicx}
\usepackage{bm}

\usepackage[utf8]{inputenc}
\usepackage[T1]{fontenc}
\usepackage{mathptmx}
\usepackage{etoolbox}
\usepackage{subcaption}
\usepackage{caption}
\usepackage[ruled,vlined]{algorithm2e}
\usepackage{amsthm}

\usepackage[ruled,vlined]{algorithm2e}
\usepackage{amsthm}
\newtheorem{thm}{Theorem}

\newtheorem{lem}{Lemma}
\newtheorem{ans}{Ansatz}

\newcommand{\beq}{\begin{equation}}
\newcommand{\eeq}{\end{equation}}
\newcommand{\ba}{\begin{array}}
\newcommand{\ea}{\end{array}}
\newcommand{\bea}{\begin{eqnarray}}
\newcommand{\eea}{\end{eqnarray}}
\newcommand{\bc}{\begin{center}}
\newcommand{\ec}{\end{center}}

\newcommand{\bt}{\begin{tabular}}
\newcommand{\et}{\end{tabular}}
\newcommand{\bi}{\begin{itemize}}
\newcommand{\ei}{\end{itemize}}
\newcommand{\bd}{\begin{description}}
\newcommand{\ed}{\end{description}}
\newcommand{\bp}{\begin{pmatrix}}
\newcommand{\ep}{\end{pmatrix}}

\newcommand{\gnorm}[1]{\left|\left| #1\right|\right|}

\newcommand{\dstream}{\left\{{\bf y}_{j}\right\}_{j=1}^{N_{T}+1}}

\title{Deep Learning Enhanced Dynamic Mode Decomposition}
\author[1,2,3]{Daniel J.  Alford-Lago\thanks{daniel.j.alford-lago.civ@us.navy.mil}}
\author[2]{Christopher W. Curtis}
\author[3]{Alexander T. Ihler}
\author[4]{Opal Issan}
\affil[1]{Naval Information Warfare Center Pacific, San Diego, CA, 92152}
\affil[2]{Department of Mathematics and Statistics, San Diego State University, San Diego, CA, 92182}
\affil[3]{Department of Computer Science,University of California Irvine, Irvine, CA, 92697}
\affil[4]{Department of Mechanical and Aerospace Engineering, University of California San Diego, San Diego, CA, 92093}

\date{}
\begin{document}
\maketitle

\begin{abstract}
Koopman operator theory shows how nonlinear dynamical systems can be represented as an infinite-dimensional, linear operator acting on a Hilbert space of observables of the system. However, determining the relevant modes and eigenvalues of this infinite-dimensional operator can be difficult. The extended dynamic mode decomposition (EDMD) is one such method for generating approximations to Koopman spectra and modes, but the EDMD method faces its own set of challenges due to the need of user defined observables. To address this issue, we explore the use of autoencoder networks to simultaneously find optimal families of observables which also generate both accurate embeddings of the flow into a space of observables and submersions of the observables back into flow coordinates. This network results in a global transformation of the flow and affords future state prediction via the EDMD and the decoder network. We call this method the deep learning dynamic mode decomposition (DLDMD). The method is tested on canonical nonlinear data sets and is shown to produce results that outperform a standard DMD approach and enable data-driven prediction where the standard DMD fails.
\end{abstract}

\section{\label{sec:level1}Introduction:\protect}
One particular collection of methods that has garnered significant attention and development is those built around the approximation of Koopman operators \cite{koopman}, which are broadly described as dynamic mode decomposition (DMD).  These methods in some sense fit within the larger context of modal decompositions \cite{wilcox, taira}, but in several ways, they go further insofar as they also readily generate proxies for the flow maps connected to the otherwise unknown but underlying dynamical systems generating the data in question.  The first papers in this area \cite{schmid, mezic1} showed across a handful of problems that with a surprisingly straightforward approach, one could readily generate accurate models with nothing but measured time series.  Further extensions soon appeared by way of the extended DMD (EDMD) and the kernel DMD (KDMD) \cite{williams, williams2} which generalized the prior approaches in such a way to better exploit the approximation methodologies used in Reproducing Kernel Hilbert Spaces. These methods were used to develop an all purpose approach to model discovery \cite{kutz2}. See \cite{kutz} for a textbook presentation and summary of what one might describe as the first wave of DMD literature.

However, in the above methods, user choices for how to represent data are necessary, and as shown in \cite{kutz3}, even what appear to be natural choices can produce misleading or even outright incorrect results. In response, researchers have brought several supervised learning, or what one might broadly call machine learning (ML), approaches to bear on the problem of finding optimal representations of data for the purpose of approximating Koopman operators. These approaches have included dictionary learning approaches \cite{bollt2, takeishi, yeung} and related kernel learning methods \cite{degennaro}. Likewise, building on theoretical results in \cite{budisic, bollt1}, more general methods seeking the discovery of optimal topological conjugacies of the flow have been developed \cite{lusch, bramburger2021deep}. Related work on discovering optimal embeddings of chaotic time series \cite{gilpin} and ML driven model discovery \cite{lusch2, mardt, erichson-2019} have also recently appeared. 

Between the dictionary learning and topological conjugacy approaches then, in this work we present a hybrid methodology that uses ML to discover optimal dictionaries of observables which also act as embeddings of the dynamics that enforce both one-time step accuracy as well as global stability in the approximation of the time series. A similar philosophy is used in \cite{azencot}. However, in that work autoencoders are used to obtain low-dimensional structure from high-dimensional data; see also the results in \cite{lusch2, Li2020Learning}. In contrast, in this paper we look at low-dimensional dynamical systems and use autoencoders to find optimal embeddings of the phase-space coordinates. Running EDMD on the embedded coordinates then generates global linearizations of the flow which are characterized by a single set of eigenvalues that govern the dynamics of a given trajectory. Our spectral results may be contrasted with that of \cite{lusch}, in which their autoencoder is paired with an auxiliary network that parameterizes new eigenvalues at each point along the trajectory. We forego this additional parameterization and obtain a more clearly global spectral representation for any given path in the flow.

To assess the validity and robustness of our approach, we explore a number of examples motivated by several different phase-space geometries. We examine how our method deals with the classic harmonic oscillator, the Duffing system, the multiscale Van der Pol oscillator, and the chaotic Lorenz-63 equations. This presents a range of problems in which one has a global center, two centers, slow/fast dynamics on an attractive limit cycle, and finally a strange attractor. As we show, our method is able to learn across these different geometries, thereby establishing confidence that it should be adaptable to more complex problems.  While accuracy for test data is high for the planar problems, as expected, the chaotic dynamics in Lorenz-63 cause prediction accuracy to degrade relatively quickly. Nevertheless, with a minimal number of assumptions and guidance in our learning algorithm, we are able to generate a reasonable approximation to the strange attractor, which shows our method should be applicable in a wider range of use cases than those explored here.

In Section \ref{sec:koopman}, we introduce the Koopman operator and outline the DMD for finding a finite-dimensional approximation to the Koopman operator. The role of deep autoencoders in our method is presented in Section \ref{sec:autoencoder}, and our algorithm and implementation details are explained in Section \ref{sec:algorithm}.  Results are presented in Section \ref{sec:results}. Discussion of results and explorations of future directions are presented in Section \ref{sec:conclusionsfuture}.

\section{\label{sec:koopman}The Koopman Operator and Dynamic Mode Decomposition}
In this work, we seek to generate approximate predictive models for time series, say $\dstream$, which are generated by an unknown dynamical system of the form
\begin{equation}\label{eqn:dynsys}
    \frac{d}{dt}{\bf y}(t) = f({\bf y}(t)), \qquad {\bf y}(0) = {\bf x} \in \mathcal{M} \subseteq \mathbb{R}^{N_{s}},
\end{equation}
where $\mathcal{M}$ is some connected, compact subset of $\mathbb{R}^{N_{s}}$.  We denote the affiliated flow of this equation as ${\bf y}(t) = \varphi(t;{\bf x})$, and {\it observables} of the state as $g({\bf y}(t))$, where $g: \mathcal{M} \mapsto \mathbb{C}$.  To build these approximate predictive models, we build on the seminal work in \cite{koopman} which shows that there exists a linear representation of the flow map given by the {\it Koopman} operator $\mathcal{K}^{t}$ where
\begin{equation} \label{eqn:koop1}
    \mathcal{K}^{t}g({\bf x}) =  g\left(\varphi(t;{\bf x})\right).
\end{equation}

Linearity is gained at the cost of turning the problem of predicting a finite-dimensional flow into one of examining the infinite-dimensional structure of the linear operator $\mathcal{K}^{t}$.  To wit, if we consider the observables $g$ to be members of the associated Hilbert space of {\it observables}, say $L_{2}\left(\mathcal{M},\mathbb{C}\right)$, $g \in L_{2}\left(\mathcal{M},\mathbb{C}\right)$ if $\gnorm{g}_{2}<\infty$ where $\gnorm{g}_{2}$ is given by 
\[
\gnorm{g}_{2}^{2} = \int_{\mathcal{M}} \left|g({\bf x})\right|^{2} d {\bf x}.
\]
For concreteness, we have chosen to integrate with respect to Lebesgue measure, though of course others could be more convenient or appropriate.  We see then that this makes the Koopman operator $\mathcal{K}^{t}$ an infinite-dimensional map such that  
\[
\mathcal{K}^{t}:L_{2}\left(\mathcal{M},\mathbb{C}\right)\rightarrow L_{2}\left(\varphi_{-t}\left(\mathcal{M}\right),\mathbb{C}\right).
\]
In this case, one has 
\begin{align*}
\gnorm{\mathcal{K}^{t}g}^{2}_{2} = & \int_{\varphi_{-t}\left(\mathcal{M}\right)}\left|g(\varphi(t;{\bf x}))\right|^{2}d{\bf x}, \\
= & \int_{\mathcal{M}}\left|g({\bf x})\right|^{2}J^{-1}(t;{\bf x})d{\bf x}, \\
\leq & \left(\sup_{{\bf x}\in\mathcal{M}}J^{-1}(t;{\bf x})\right) \gnorm{g}^{2}_{2},
\end{align*}
where 
\[
J(t;{\bf x}) = \left|\mbox{det}\left(D_{{\bf x}}\varphi(t;{\bf x})\right)\right|.
\]
Based on the above computations, we observe that the Koopman operator is an isometry for volume preserving flows.  

If we further suppose that $\mathcal{M}$ is invariant with respect to the flow, or $\varphi_{t}(\mathcal{M})\subset \mathcal{M}$ for $t>0$, we can simplify the above statement so that 
\[
\mathcal{K}^{t}:L_{2}\left(\mathcal{M},\mathbb{C} \right)\rightarrow L_{2}\left(\mathcal{M},\mathbb{C} \right).
\]
We will assume this requirement going forward since it is necessary to make much in the way of concrete statements about the analytic properties of the Koopman operator.  In particular, using Equation \eqref{eqn:koop1}, if we further suppose that we have an observable $g$ such that $g\in C_{1}(\mathcal{M})$ then we see that 
\[
\lim_{t\rightarrow 0^{+}} \frac{\gnorm{\mathcal{K}^{t}g - g}_{\infty}}{t} = \gnorm{f\cdot \nabla g}_{\infty},
\]
where $\gnorm{\cdot}_{\infty}$ denotes the supremum norm over the compact subset $\mathcal{M}$.  From this computation, we find that the infinitesimal generator, say $\mathcal{L}$, affiliated with $\mathcal{K}^{t}$ is given by
\[
\mathcal{L}g = f({\bf x})\cdot \nabla g\,,
\]
and we likewise can use the Koopman operator to solve the initial-value problem 
\[
u_{t} = \mathcal{L}u, ~ u({\bf x},0) = g({\bf x}),
\]
so that $u({\bf x},t)=\mathcal{K}^{t}g({\bf x})$.  From here, the central question in Koopman analysis is to determine the spectrum and affiliated modes of $\mathcal{L}$ since these then completely determine the behavior of $\mathcal{K}^{t}$.  This can, of course, involve solving the classic eigenvalue problem 
\[
\mathcal{L}\phi = \lambda \phi.
\]
However, as the reader may have noticed, there has been no discussion of boundary conditions.  Therefore, while one can get many useful results focusing on this problem, one must also allow that continuous spectra are a natural feature, and eigenfunctions can have pathological behavior as well as being difficult to categorize completely and in detail; see \cite{bollt1} and \cite{mezic4}.  As this issue represents an evolving research area, we do not attempt to make any further claim and simply follow in the existing trend of the literature and focus on trying to numerically solve the classic eigenvalue problem.  

To this end, if we can find the Koopman eigenfunctions $\{\phi_l\}_{l=1}^{\infty}$ with affiliated eigenvalues $\{\lambda_l \}_{l=1}^{\infty}$, where
\begin{align*}
    \mathcal{K}^{t}\phi_{l} = e^{t\lambda_{l}}\phi_{l}, \quad l\in \{1,2,\dots \}\,,
\end{align*}
then for any observable $g$ one, in principle, has a modal decomposition such that 
\[
    g({\bf x}) = \sum_{l=1}^{\infty}c_{l}\phi_{l}({\bf x})\,,
\]
as well as an analytic representation of the associated dynamics
\begin{equation} \label{eqn:koop2}
    \mathcal{K}^{t}g({\bf x}) = \sum_{l=1}^{\infty}c_{l}e^{t\lambda_{l}}\phi_{l}({\bf x})\,.
\end{equation}
Note, that the analytic details we provide regarding the Koopman operator in this section are those we find most essential to understanding the present work and for providing a reasonably complete reading experience.  Essentially all of it exists, and more detail and depth can be found, in \cite{lasota, mezic5, gonzalez2021antikoopmanism}, among many other sources.  

Now, the challenge of determining the modes and eigenvalues of the infinite-dimensional operator, $\mathcal{K}^{t}$, remains. In general, this is impossible to obtain in an analytic way, however, the DMD and its extensions,  the EDMD and the KDMD \cite{schmid,mezic1, kutz, williams, williams2}, allow for the numerical determination of a finite number of the Koopman modes and eigenvalues.  In this work, we focus on the EDMD since it most readily aligns with the methodologies of ML.  

The EDMD begins with the data set $\dstream$ where
\[
{\bf y}_{j} = \varphi(t_{j};{\bf x}), ~t_{j} =(j-1)\delta t\,,
\]
where $\delta t$ is the time step at which data are sampled.  Note, per this definition, we have that ${\bf y}_{1}={\bf x}$.  Following \cite{williams, williams2}, given our time snapshots $\dstream$, we suppose that any observable $g({\bf x})$ of interest lives in a finite-dimensional subspace $\mathcal{F}_{D}\subset L_{2}\left(\mathcal{O}\right)$ described by a given basis of observables $\left\{\psi_{l}\right\}_{l=1}^{N_{o}}$ so that 
\begin{align}\label{eqn:g_psi}
g({\bf x}) = \sum_{l=1}^{N_{o}}a_{l}\psi_{l}\left({\bf x}\right)\,.
\end{align}
Given this ansatz, we then suppose that  
\begin{align}\label{eqn:koopman_phi}
\mathcal{K}^{\delta t}g({\bf x}) = &\sum_{l=1}^{N_{o}}a_{l}\psi_{l}\left(\varphi\left(\delta t, {\bf x}\right) \right) \\
 = & \sum_{l=1}^{N_{o}}\psi_{l}({\bf x})\left({\bf K}^{T}_{o}{\bf a} \right)_{l} + r({\bf x};{\bf K}_{o})\, \nonumber
\end{align}
where $r({\bf x};{\bf K}_{o})$ is the associated error which results from the introduction of the finite-dimensional approximation of the Koopman operator represented by the $N_{o}\times N_{o}$ matrix ${\bf K}_{o}$.  

While we ultimately find a matrix ${\bf K}_{o}$ to minimize the error $r({\bf x};{\bf K}_{o})$ relative to the choice of observables, for this approach to make any real computational sense, we tacitly make the following assumption when using the EDMD
\begin{ans}\label{ansz1}
We have chosen observables $\left\{\psi_{l}\right\}_{l=1}^{N_{o}}$ such that the space 
\[
\mathcal{F}_{D} = \mbox{Span}\left(\left\{\psi_{l}\right\}_{l=1}^{N_{o}} \right) 
\]
is invariant under the action of the Koopman operator $\mathcal{K}^{\delta t}$, i.e. 
\[
\mathcal{K}^{\delta t} \mathcal{F}_{D} \subset \mathcal{F}_{D}\,.
\]
Equivalently, we suppose that there exists a set of observables for which the affiliated error $r({\bf x};{\bf K}_{o})$ in Equation \eqref{eqn:g_psi} is identically zero.  
\end{ans}
\noindent We see that if this condition holds for $\mathcal{K}^{\delta t}$, then it also holds for $\mathcal{K}^{n\delta t}$, where $n$ is an integer such that $n\geq1$; therefore, this Ansatz is stable with respect to iteration of the discrete time Koopman operator.  If this condition does not hold, or at least hold up to some nominal degree of error, then one should not imagine that the  EDMD method is going to provide much insight into the behavior of the Koopman operator.  

One can then demonstrate, in line with the larger DMD literature, that finding the error minimizing matrix ${\bf K}_{o}$ is equivalent to solving the optimization problem 
\begin{equation}\label{eqn:dmd1}
    {\bf K}_{o} = \underset{{\bf K}}{\mathrm{argmin}}\gnorm{\bm{\Psi}_{+} - {\bf K}\bm{\Psi}_{-}}_{F}^{2}\,,
\end{equation}
where the $N_{o}\times N_{T}$ matrices $\bm{\Psi}_{\pm}$ are given by 
\begin{equation}\label{eqn:dmd_psi}
\bm{\Psi}_{-} = \left\{\bm{\Psi}_{1}~ \bm{\Psi}_{2}~ \cdots ~\bm{\Psi}_{N_{T}} \right\}, \quad \boldsymbol{\Psi}_{+} = \left\{\bm{\Psi}_{2} ~\bm{\Psi}_{3} ~\cdots ~\bm{\Psi}_{N_{T}+1} \right\}\,,
\end{equation}
where each column in the above matrices is a $N_{o}\times 1$ vector of observables of the form 
\[
\bm{\Psi}_{j} = \left(\psi_{1}({\bf y}_{j}) ~ \cdots ~ \psi_{N_{o}}({\bf y}_{j}) \right)^{T},
\]
where $\gnorm{\cdot}_{F}$ is the Frobenius norm.  In practice, we solve \eqref{eqn:dmd1} using the Singular Value Decomposition (SVD) of $\bm{\Psi}_{-}$ so that
\begin{align*}
    \bm{\Psi}_{-} = {\bf U} \bm{\Sigma} {\bf W}^\dagger\,.
\end{align*}
This then gives us 
\begin{equation*} \label{eqn:dmd2}
    {\bf K}_{o} = \bm{\Psi}_{+} {\bf W} \bm{\Sigma}^{-P} {\bf U}^{\dagger}\,, 
\end{equation*} 
where $-P$ denotes the Moore--Penrose pseudoinverse.  The corresponding error in the Frobenius norm $E_{r}(\bm{K}_{o})$ is given by 
\[
E_{r}({\bf K}_{o}) = \gnorm{\bm{\Psi}_{+}\left({\bf I} - {\bf W}{\bf W}^{\dagger} \right)}_{F}\,.
\]
We see that $E_{r}({\bf K}_{o})$ serves as a proxy for the error function $r({\bf x};{\bf K}_{o})$.

Following existing methodology \cite{williams},  if we wished to find Koopman eigenfunctions and eigenvalues, then after diagonalizing ${\bf K}_{o}$ so that 
\[
{\bf K}_{o} = {\bf V} {\bf T} {\bf V}^{-1}, ~{\bf T}_{ll} = \tilde{t}_{l}\,, 
\]
then one can show that the quantity $\lambda_{l} = \ln(\tilde{t}_{l})/\delta t$ should be an approximation to an actual Koopman eigenvalue and 
\begin{equation}\label{eqn:emodes}
\phi_{l}({\bf x}) = \sum_{m=1}^{N_{o}}\psi_{m}({\bf x}){\bf V}^{-1}_{lm}, ~ l=1,\cdots,N_{o}\,.
\end{equation}
From here, in the traditional EDMD algorithm, one approximates the dynamics via the reconstruction formula  
\[
{\bf y}(t;{\bf x})\approx \sum_{l=1}^{N_{o}}{\bf k}_{l}e^{t\lambda_{l}}\phi_{l}({\bf x})\,,
\]
where the {\it Koopman modes} ${\bf k}_{l}\in \mathbb{C}^{N_{s}}$ in principle solve the initial-value problem 
\[
{\bf x} = \sum_{l=1}^{N_{o}}{\bf k}_{l} \phi_{l}({\bf x})\,.  
\]
In the original coordinates of the chosen observables, using Equation \eqref{eqn:emodes} we obtain the equivalent formula 
\[
{\bf y}(t;{\bf x})\approx {\bf K}_{m}e^{t\bm{\Lambda}}{\bf V}^{-1}\bm{\Psi}({\bf x})\,,
\]
where ${\bf K}_{m}$ is the $N_{s}\times N_{o}$ matrix whose columns are the Koopman modes ${\bf k}_{l}$ and $\bm{\Lambda}$ is the $N_{o}\times N_{o}$ diagonal matrix with diagonal entries $\lambda_{l}$.  In practice, we find the Koopman modes via the fitting formula
\begin{equation}\label{eqn:edmdfit}
{\bf K}_{m} = \underset{{\bf H}}{\mathrm{argmin}} \sum_{j=1}^{N_{T}+1}\gnorm{{\bf y}_{j}-{\bf H}{\bf V}^{-1}\bm{\Psi}({\bf y}_{j})}_{2}\,,
\end{equation}
where ${\bf H}$ is any complex $N_{s}\times N_{o}$ matrix and the norm $\gnorm{\cdot}_{2}$ refers to the standard Euclidean norm.  Of course, the appropriateness of this fit is completely contingent on the degree to which Ansatz 1 holds.  We address this issue in Section \ref{sec:autoencoder}.  

Finally, we note that the standard DMD is given by letting $N_{o}=N_{s}$ and $\psi_{l}({\bf x})=x_{l}$.  In this case, we have that ${\bf K}_{m}={\bf V}$.  Due to its popularity and ease of implementation, we use the DMD for reference in Section \ref{sec:results}.

\section{\label{sec:autoencoder}The Deep Learning Dynamic Mode Decomposition}
The central dilemma when using the EDMD is finding suitable collections of observables relative to the data stream $\dstream$.  To accomplish this in an algorithmic way, we suppose that an optimal set of observables can be found using a deep neural network $\mathcal{E}$ where 
\[
\mathcal{E}: \mathbb{R}^{N_{s}} \rightarrow \mathbb{R}^{N_{o}}, ~ \mathcal{E}\left(\bm{x}\right) = \tilde{\bm{x}}\,,
\]
and where across dimensions we define 
\[
\tilde{x}_{l} = \mathcal{E}_{l}({\bf x}), ~ l=1,\cdots,N_{o}\,.
\]
We call this the {\it encoder} network and the transformed coordinates, $\tilde{{\bf x}}$, the {\it latent variables}. Given that the neural network representing $\mathcal{E}$ consists of almost everywhere smooth transformations, we note that by Sard's Theorem \cite{abraham}, we can generically assume that $\mathcal{E}$ has a Jacobian of full rank at almost all points in the domain of the encoder.  We generically assume that the latent dimension $N_{o}\geq N_{s}$, making $\mathcal{E}$ an immersion \cite{abraham} from the flow space into the space of observables.  In this case, the matrices $\bm{\Psi}_{\pm}$ from equation \eqref{eqn:dmd_psi} are now given by 
\[
\bm{\Psi}_{-} = \left\{\tilde{\bm{y}}_{1} ~ \cdots ~ \tilde{\bm{y}}_{N_{T}}\right\}, ~ \bm{\Psi}_{+} = \left\{\tilde{\bm{y}}_{2} ~ \cdots ~ \tilde{\bm{y}}_{N_{T}+1}\right\}, ~\tilde{{\bf y}}_{j} = \mathcal{E}({\bf y}_{j})\,,
\]
so that we compute ${\bf K}_{o}$ in the latent-space coordinates .  Beyond the lower bound given above, we do not enforce any constraint on the number of dimensions of the latent space. Instead, $N_{o}$ is treated as a hyperparameter. This is a key feature of our encoder networks as they are allowed to lift the data into higher-dimensional latent spaces, giving the EDMD a rich, and flexibly defined, space of observables with which to work.   

Corresponding to the immersion $\mathcal{E}$, we introduce the {\it decoder} network $\mathcal{D}$ which is meant to act as the inverse of $\mathcal{E}$, i.e. we seek a decoder mapping acting as a submersion
\[
\mathcal{D}:\mathbb{R}^{N_{o}} \rightarrow \mathbb{R}^{N_{s}}\,,
\] 
so that 
\begin{equation}\label{eqn:ED_def}
\mathcal{D}\circ\mathcal{E}\left({\bf x} \right)= {\bf x}\,.
\end{equation}
Upon finding such a mapping, we can show the following lemma.
\begin{lem}
If for immersion $\mathcal{E}: \mathbb{R}^{N_{s}} \rightarrow \mathbb{R}^{N_{o}}$ there exists corresponding submersion $\mathcal{D}:\mathbb{R}^{N_{o}} \rightarrow \mathbb{R}^{N_{s}}$ such that 
\[
\mathcal{D}\circ\mathcal{E}\left({\bf x} \right)= {\bf x}\,,
\]
then $\mathcal{E}$ is injective and therefore an embedding.  
\end{lem}
\begin{proof}
Suppose 
\[
\mathcal{E}({\bf x}_{1}) = \mathcal{E}({\bf x}_{2})\,.
\]
Then we see by identity that 
\[
\mathcal{D}\circ\mathcal{E}({\bf x}_{1}) = \mathcal{D}\circ\mathcal{E}({\bf x}_{2})\,,
\]
and therefore ${\bf x}_{1}={\bf x}_{2}$.  Thus $\mathcal{E}$ is an embedding.  
\end{proof}

As shown in \cite{mezic5,bollt1}, flows which are diffeomorphic to one another share Koopman eigenvalues and have eigenfunctions that are identical up to diffeomorphism.  In our case, $\mathcal{E}$ and $\mathcal{D}$ are typically not invertible since the reverse composition, $\mathcal{E}\circ\mathcal{D}$, does not necessarily yield the identity.  However, we can define the affiliated flow $\tilde{\varphi}(t;\tilde{{\bf x}})$ such that 
\[
\tilde{\varphi}(t;\tilde{{\bf x}}) = \mathcal{E}\left(\varphi(t;{\bf x})\right)\,.
\]
Immediately, we find that there must of course be an affiliated Koopman operator, $\tilde{\mathcal{K}}^{t}$, corresponding to this encoded, or lifted, flow. This then allows us to show the following theorem.
\begin{thm}
With $\mathcal{E}$ and $\mathcal{D}$ as above, if $(\phi({\bf x}),e^{\lambda t})$ are a spectral pair for the Koopman operator $\mathcal{K}^{t}$, then $(\tilde{\phi}(\tilde{{\bf x}}),e^{\lambda t})$ are a spectral pair for the Koopman operator $\tilde{\mathcal{K}}^{t}$ where $\tilde{\phi}=\phi\circ\mathcal{D}$.
\end{thm}
\begin{proof}
For the operator $\mathcal{K}^{t}$, if we suppose it has eigenfunction $\phi({\bf x})$ with corresponding eigenvalue $e^{\lambda t}$, then we have the identities
\[
\mathcal{K}^{t}\phi({\bf x}) = e^{\lambda t}\phi({\bf x}) = \phi(\varphi(t;{\bf x}))\,.
\]
We then have the affiliated identities 
\begin{align*}
\phi(\varphi(t;{\bf x})) &= \left(\phi \circ \mathcal{D}\right) \left(\tilde{\varphi}(t;\tilde{{\bf x}})\right)\\
 &= e^{\lambda t} \left(\phi \circ \mathcal{D}\right)\left( \tilde{{\bf x}} \right) \\
 &= \tilde{\mathcal{K}}^{t}\left(\phi \circ \mathcal{D}\right)\left( \tilde{{\bf x}} \right)\,,
\end{align*}
and so in particular we see that 
\[
\tilde{\mathcal{K}}^{t}\left(\phi \circ \mathcal{D}\right)\left( \tilde{{\bf x}} \right) = e^{\lambda t} \left(\phi \circ \mathcal{D}\right)\left( \tilde{{\bf x}} \right)\,.  
\]
\end{proof}

Thus we see that every Koopman mode and eigenvalue in the original flow space is lifted up by the embedding $\mathcal{E}$.  That said, it is of course possible then that new spectral information affiliated with $\tilde{\mathcal{K}}^{t}$ can appear, and if we perform the EDMD in the lifted variables, there is no immediate guarantee we are, in fact, computing spectral information affiliated with the primary Koopman operator $\mathcal{K}^{t}$.  That said, if in fact Ansatz 1 holds for the given choice of observables, which is to say we have made the right choice of embedding $\mathcal{E}$, then we can prove the following theorem.
\begin{thm}
Assuming Ansatz 1 holds relative to some choice of observables \\ $\left\{\psi_{l}({\bf x})\right\}_{l=1}^{N_{o}}$, suppose that the action of $\mathcal{K}^{\delta t}$ on each observable is given by the $N_{o}\times N_{o}$ connection matrix ${\bf C}(\delta t)$ where
\[
\mathcal{K}^{\delta t}\psi_{m}({\bf x}) = \sum_{l=1}^{N_{o}}\text{C}_{ml}(\delta t) \psi_{l}({\bf x})\,,
\]
with $\text{C}_{ml} $ denoting the entries of the connection matrix.  If ${\bf C}(\delta t)$ is diagonalizable, then the EDMD algorithm only computes spectra and eigenfunctions of $\mathcal{K}^{t}$, $t>0$.  
\end{thm}
\begin{proof}
For $g\in \mathcal{F}_{D}$, we have that 
\[
g({\bf x}) = \sum_{l=1}^{N_{o}}a_{l}\psi_{l}({\bf x})\,,
\]
If Ansatz 1 holds, it is then the case that 
\[
\mathcal{K}^{\delta t}\left(\sum_{l=1}^{N_{o}}a_{l}\psi_{l}({\bf x})\right) = \sum_{m=1}^{N_{o}}b_{m}\psi_{m}({\bf x}),
\]
and likewise we must have that 
\begin{align*}
\mathcal{K}^{\delta t}\left(\sum_{l=1}^{N_{o}}a_{l}\psi_{l}({\bf x})\right) = & \sum_{l=1}^{N_{o}}a_{l}\mathcal{K}^{\delta t}\psi_{l}({\bf x})\\
= & \sum_{l=1}^{N_{o}}\sum_{m=1}^{N_{o}}a_{l}C_{lm}(\delta t)  \psi_{m}({\bf x})\,.
\end{align*}
Thus, the action of the Koopman operator is now recast in terms of the following matrix problem 
\[
{\bf C}^{T}(\delta t) {\bf a} = {\bf b}\,.
\]
Likewise, if we ask for the corresponding connection matrix of higher powers of $\mathcal{K}^{\delta t}$ so that for integer $n\geq 1$ we have 
\[
\mathcal{K}^{n\delta t}\left(\sum_{l=1}^{N_{o}}a_{l}\psi_{l}({\bf x})\right) = \sum_{l=1}^{N_{o}}\sum_{m=1}^{N_{o}}a_{l}C_{lm}(n\delta t)  \psi_{m}({\bf x})\,.
\]
Then we see that 
\[
{\bf C}\left(n\delta t\right) = {\bf C}\left(\delta t\right){\bf C}\left((n-1)\delta t\right)\,,
\]
or, defining ${\bf C}(0)={\bf I}$ with ${\bf I}$ being the $N_{o}\times N_{o}$ identity matrix,
\[
{\bf C}\left(n\delta t\right) = {\bf C}^{n}\left(\delta t\right).
\]
Moreover, referring to the EDMD algorithm, clearly ${\bf C}(\delta t)={\bf K}_{o}$, so that if ${\bf K}_{o}={\bf V}{\bf T}{\bf V}^{-1}$, then 
\[
{\bf C}\left(n\delta t\right) = {\bf V}{\bf T}^{n}{\bf V}^{-1}\,.
\]

Choosing then the vector of coefficients ${\bf a}$ to be the $j^{\text{th}}$ column of $({\bf V}^{-1})^{T}$ or ${\bf a} = ({\bf V}^{-1})^{T}_{j}$, we see 
\[
\mathcal{K}^{\delta t}\left(\sum_{l=1}^{N_{o}}a_{l}\psi_{l}({\bf x})\right) = e^{\delta t \lambda_{j}}\left(\sum_{l=1}^{N_{o}}a_{l}\psi_{l}({\bf x})\right),
\]
so that we have using the EDMD algorithm we see we have computed $N_{o}$ eigenvalues and eigenvectors of $\mathcal{K}^{\delta t}$.  Passing to the infinitesimal generator $\mathcal{L}$ allows us to then extend the result for the Koopman operator $\mathcal{K}^{t}$ for $t>0$.    
\end{proof}

So we see on the one hand that the EDMD left to its own devices is prone to introducing perhaps spurious spectral information, and of course, without recourse to a known reference, we have no way in advance of knowing how to tell which results generated via the EDMD produce relevant spectra.  We note that this issue was numerically illustrated in \cite{kutz3}.  On the other hand, if we can somehow ensure Ansatz 1 holds, then the EDMD is guaranteed to produce meaningful results with essentially zero error.  Of course, what remains in either case is the fundamental dilemma of how to choose observables such that Ansatz 1 is enforced, or at least such that the error is guaranteed to be controlled in some uniform way.  

To address this dilemma then, we propose the deep learning dynamic mode decomposition (DLDMD), in which we determine the encoder/decoder pair $\mathcal{E}$ and $\mathcal{D}$ by minimizing the following loss function $\mathcal{L}$, where 
\begin{equation}\label{eqn:lossfun}
	\mathcal{L} = \alpha_1\mathcal{L}_{\text{recon}} + \alpha_2\mathcal{L}_{\text{dmd}} + \alpha_3\mathcal{L}_{\text{pred}} + \alpha_4||{\bf W}_{g}||^{2}_{2}\,,
\end{equation}
such that 
\begin{align*}\label{eqn:loss}
    \mathcal{L}_{\text{recon}} &= \frac{1}{N_{T}+1}\sum_{j=1}^{N_{T}+1} ||\textbf{y}_{j}-\mathcal{D}(\mathcal{E}(\textbf{y}_{j}))||_{2}\,, \\[2ex]
    \mathcal{L}_{\text{dmd}}  &= E_{r}\left({\bf K}_{o}\right)\,,\\[2ex]
    \mathcal{L}_{\text{pred}} &= \frac{1}{N_{T}}\sum_{j=1}^{N_{T}} ||\textbf{y}_{j+1} - \mathcal{D}({\bf V} {\bf T}^j {\bf V}^{-1} \mathcal{E}(\textbf{x}))||_{2}\,,
\end{align*}
and ${\bf W}_{g}$ contains the weights of the $\mathcal{E}$ and $\mathcal{D}$ networks making the final term in $\mathcal{L}$ a regularization term. The hyperparameters $\alpha_{1}, \alpha_{2}, \alpha_{3},$ and $\alpha_{4}$ provide appropriate scalings for each loss component. The autoencoder reconstruction loss,  $\mathcal{L}_{\text{recon}}$,  demands that $\mathcal{D}$ approximates the inverse of $\mathcal{E}$. We see then that $\mathcal{L}_{\text{recon}}$ in effect replaces Equation \eqref{eqn:edmdfit}.  The DMD loss, $\mathcal{L}_{\text{dmd}}$, keeps $r({\bf x};{\bf K}_{o})$ as small as possible relative to advancing one timestep.  In contrast, we see $\mathcal{L}_{\text{pred}}$ makes the overall method remain stable under repeated applications of the finite-dimensional approximation to $\mathcal{K}^{\delta t}$.  Thus $\mathcal{L}_{\text{dmd}}$ plays the role of ensuring we have a consistent time-stepping scheme, while $\mathcal{L}_{\text{pred}}$ ensures we have a globally stable method, and so by combining the two, we address the fundamental dilemma presented by Ansatz 1.  It is this ensured global stability that motivates us to call the diagonalization of the matrix ${\bf K}_{o} = {\bf V} {\bf T} {\bf V}^{-1}$ from our algorithm's EDMD the {\it global linearization} of the flow.  

We further see that this loss function implements the diagram in Figure \ref{fig:comm_diag}. Through this diagram we impose a one-to-one mapping between trajectories in the original coordinates, $\bm{y}$,  and trajectories in the latent space, $\tilde{\bm{y}}$.  Note that this diagram allows us to circumvent the unknown flow map, $\varphi(\delta t; \bm{y}_{j})$, which is equivalent to the Koopman operator of interest, $\mathcal{K}^{\delta t}$, from Equations \eqref{eqn:g_psi} and \eqref{eqn:koopman_phi}.
\begin{figure}
    \centering
    \includegraphics{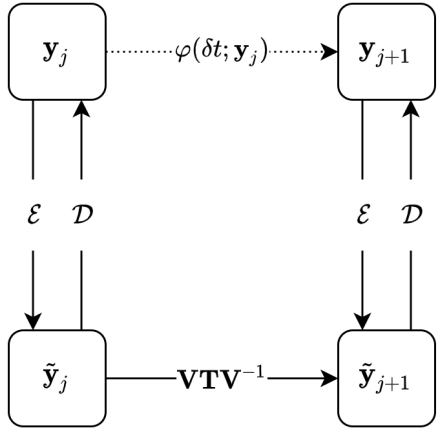}
    \caption{Diagram illustrating the relationships between the encoder ($\mathcal{E}$), decoder ($\mathcal{D}$), and EDMD/global linearization (${\bf VTV^{-1}}$) steps.  Assuming Ansatz \ref{ansz1} holds, these relations are exact. In practice, these are approximations and so we must view the mappings with solid lines as having some affiliated error in a process that allows us to circumvent not knowing the flow map $\varphi(\delta t;\bm{y}_{j})$.}
    \label{fig:comm_diag}
\end{figure}

\section{\label{sec:algorithm}The DLDMD Algorithm: Implementation Details}
We build the autoencoder in the Python programming language using Tensorflow version 2.  The deep neural networks we construct that act as $\mathcal{E}$ and $\mathcal{D}$ from Equation \eqref{eqn:ED_def} must transform each vector of coordinates along a sample trajectory to and from the latent-space coordinates , respectively.  We chose to use dense layers in each network,  though other layer types should suffice so long as they encode each point along the trajectory separately,  are densely connected, and output the correct dimensions.   

As is the case when training any type of neural network, there are a number of hyperparameters that the researcher must take care in selecting. However, we found that the encoder and decoder networks did not require significantly different hyperparameters from dataset to dataset. Notably, the architecture of the neural networks remained identical across all examples. We found that 3 hidden layers each with 128 neurons were sufficient for all of the test problems presented in Section \ref{sec:results}. The primary tunable parameter was the dimension of the latent space, $N_{o}$, which was tuned manually. We used Rectified Linear Units (ReLU) for the activation functions and chose the Adam optimizer with a learning rate of $10^{-3}$ for the harmonic oscillator, and $10^{-4}$ for Duffing, Van der Pol, and Lorenz 63. For the loss function hyperparameters in Equation \eqref{eqn:lossfun}, $\alpha_{1}, \alpha_{2}$,  and $\alpha_{3}$ were all set to $1$,  while $\alpha_{4}=10^{-9}$. The harmonic oscillator was trained with a batch size of 512 while all other systems had batch sizes of 256. All systems were trained for 1,000 epochs. The hardware used was an Nvidia Tesla V100.   

See Algorithm \ref{alg:dldmd} for the complete pseudocode for the DLDMD training method. The trained DLDMD model is applied by sending a trajectory through the encoder network, performing the EDMD using the encoded coordinates as observables, then using the modes, eigenvalues, and eigenfunctions to reconstruct the full length of the trajectory and beyond in the latent space. The decoder network then allows us to map the entire EDMD reconstruction back into the original coordinate system.  

\begin{algorithm}
\DontPrintSemicolon
\KwData{${\bf Y} \in \mathbb{R}^{n \times m}$ such that each column, $\bm{y}_i \in \mathbb{R}^{n}$, is an observation of the state variables $\delta t$ time from $\bm{y}_{i-1}$.}
\KwResult{$\mathcal{E}, \mathcal{D}, {\bf T}, {\bf V},{\bf k}$}
Initialize: set reconstruction weight $\alpha_{1}>0$, DMD prediction weight $\alpha_{2}>0$,  phase space prediction weight $\alpha_{3}>0$, and regularization weight $\alpha_{4}>0$.\;
\Begin{
	\For{$epoch = 1 \dots maxEpochs$}{
		$\bm{\Psi} \longleftarrow \mathcal{E}({\bf Y})$\;
		${\bf \bar{Y}} \longleftarrow \mathcal{D}(\bm{\Psi})$\;
		$\bm{\Psi}_{-} \longleftarrow \left[\bm{\psi}_{1}~\bm{\psi}_{2}~\cdots\bm{\psi}_{m-1} \right]$\;
		$\bm{\Psi}_{+} \longleftarrow \left[\bm{\psi}_{2}~\bm{\psi}_{3}~\cdots\bm{\psi}_{m} \right]$\;
		${\bf U}, \bm{\Sigma}, {\bf W}^{\dagger} \longleftarrow \text{SVD}(\bm{\Psi}_{-})$\;
		${\bf K}  \longleftarrow \bm{\Psi}_{+} {\bf W} \bm{\Sigma}^{-1} {\bf U}^{\dagger}$\;
		${\bf T},  {\bf V} \longleftarrow \text{EVD}({\bf K})$\;
		${\bf k} \longleftarrow \text{IVP}({\bf V},\bm{\Psi}_{-})$\;
		\For{i = 1 \dots m}{
			$\bm{\hat{\psi}}_{i} \longleftarrow {\bf V} \bm{\Sigma}^{i} {\bf k}$\;
		}
		$\bm{\hat{\Psi}} \longleftarrow \left[ \bm{\hat{\psi}}_{1} ~ \bm{\hat{\psi}}_{2} ~ \cdots \bm{\hat{\psi}}_{m} \right]$\;
		${\hat{\bf Y}} \longleftarrow \mathcal{D}(\boldsymbol{\hat{\Psi}})$\;
		$\mathcal{L} \longleftarrow \alpha_{1}||{\bar{\bf Y}} - {\bf Y}||_{\text{MSE}} + \alpha_{2} ||\bm{\Psi}_{+} ({\bf I} - {\bf W}{\bf W}^{\dagger})||_{F}$ $+ \alpha_{3} ||{\hat{\bf Y}} - {\bf Y}||_{\text{MSE}} + \alpha_{4}||{\bf W}_{g}||_{2}^{2}$\;
		$\mathcal{E}, \mathcal{D} \longleftarrow \text{OPT}(\mathcal{L})$\;
	}
}
Where SVD$(\cdot)$ is the Singular Value Decomposition, EVD$(\cdot)$ is the eigenvalue decomposition,  IVP$(\cdot,\cdot)$ solves an initial value problem, and OPT$(\cdot)$ is an appropriate optimizer for the neural networks $\mathcal{E}$ and $\mathcal{D}$. MSE indicates the mean squared error.
\caption{DLDMD\label{alg:dldmd}}
\end{algorithm}

\section{\label{sec:results}Results}
We test the DLDMD method on several datasets generated from dynamical systems that each exhibit some unique flow feature. In Sections \ref{sec:harmonic} - \ref{sec:vdp}, we examine much of the range of planar dynamics by way of studying the harmonic oscillator, Duffing, and Van der Pol equations.  For example, when limited to the first separatrix, the harmonic oscillator gives closed orbits about a single center. Proceeding in complexity, the Duffing equation is comprised of not only closed orbits but also a homoclinic connection that separates two distinct regions of the phase space, requiring now that the DLDMD compute trajectories on either side of a separatrix. The Van der Pol oscillator has trajectories both growing and decaying onto a limit cycle and exhibits multi-scale, slow/fast dynamics. Finally, in Section \ref{sec:lorenz}, we demonstrate the DLDMD on chaotic trajectories from the Lorenz-63 system, which extends our approach to three-dimensional results evolving over a strange attractor.

The training, validation, and test data for each example were generated using a fourth-order Runge-Kutta scheme. To generate test-data, for the harmonic oscillator and the Duffing system, we chose a time step of $\delta t=0.05$ and ran simulations out to $t_{f}=20$ in order to get closed orbits for all initial conditions. The Van der Pol system used $\delta_t=0.02$ and $t_f=15$. This system required a shorter integration step size to sufficiently sample the slow and fast parts of each trajectory. Finally, the Lorenz-63 system used $\delta_t=0.01$ and $t_{f}=3$. For testing, we applied the trained encoder and decoder with an EDMD decomposition on the latent trajectories and generated reconstructions up to the simulation times used for training. Then, in each case we ran the trajectories out further in time in their latent-space coordinates  using only the spectral information from their respective EDMD to evaluate stability. The prediction time for the harmonic oscillator and Duffing system was from $t_{f}=20$ to $t_{f}=40$, for the Van der Pol equation it was $t_{f}=15$ to $t_{f}=30$, and finally for the Lorenz-63 equations, it was from $t_{f}=3$ to $t_{f}=6$.  

We generate 15,000 trajectories for each system, using 10,000 for training, 3,000 for validation, and 2,000 for testing. Each trajectory is generated by uniform random sampling of initial conditions from some pre-defined region of phase space. For the harmonic oscillator, we used $x_{1} \in \{-3.1, 3.1\}$, $x_{2}\in\{-2,2\}$, and we limited our choice of trajectories to those within the first separatrix using the potential function $0.5x_{2}^{2} - \cos{x_{1}} < 0.99$. The Duffing system was generated over initial conditions sampled from $x_{1} \in \{-1,1\}$ and $x_{2} \in \{-1,1\}$. The Van der Pol system used $x_{1} \in \{-2, 2\}$ and $x_{2} \in \{-2, 2\}$ to generate trajectories and was then scaled by its standard deviation. The Lorenz-63 system used $x \in \{-15, 15\}$, $y \in \{-20, 20\}$, and $z \in \{0, 40\}$.

\subsection{The Harmonic Oscillator: One Center}\label{sec:harmonic}
The first system we consider is a nonlinear oscillator described by the undamped pendulum system,
\begin{align*}
	\dot{x}_{1} &= x_{2}, \\[1ex]
	\dot{x}_{2} &= -\mathrm{sin}(x_{1}).
\end{align*}
This system exhibits nearly linear dynamics near the origin and becomes increasingly nonlinear towards the separatrix. We limited the dataset to just those trajectories that lie below the separatrix in order to test the DLDMD on a system with only closed Hamiltonian orbits about a single center. 

\begin{figure}
    \hspace*{-1.5cm}
    \centering
    \includegraphics[width=0.65\linewidth]{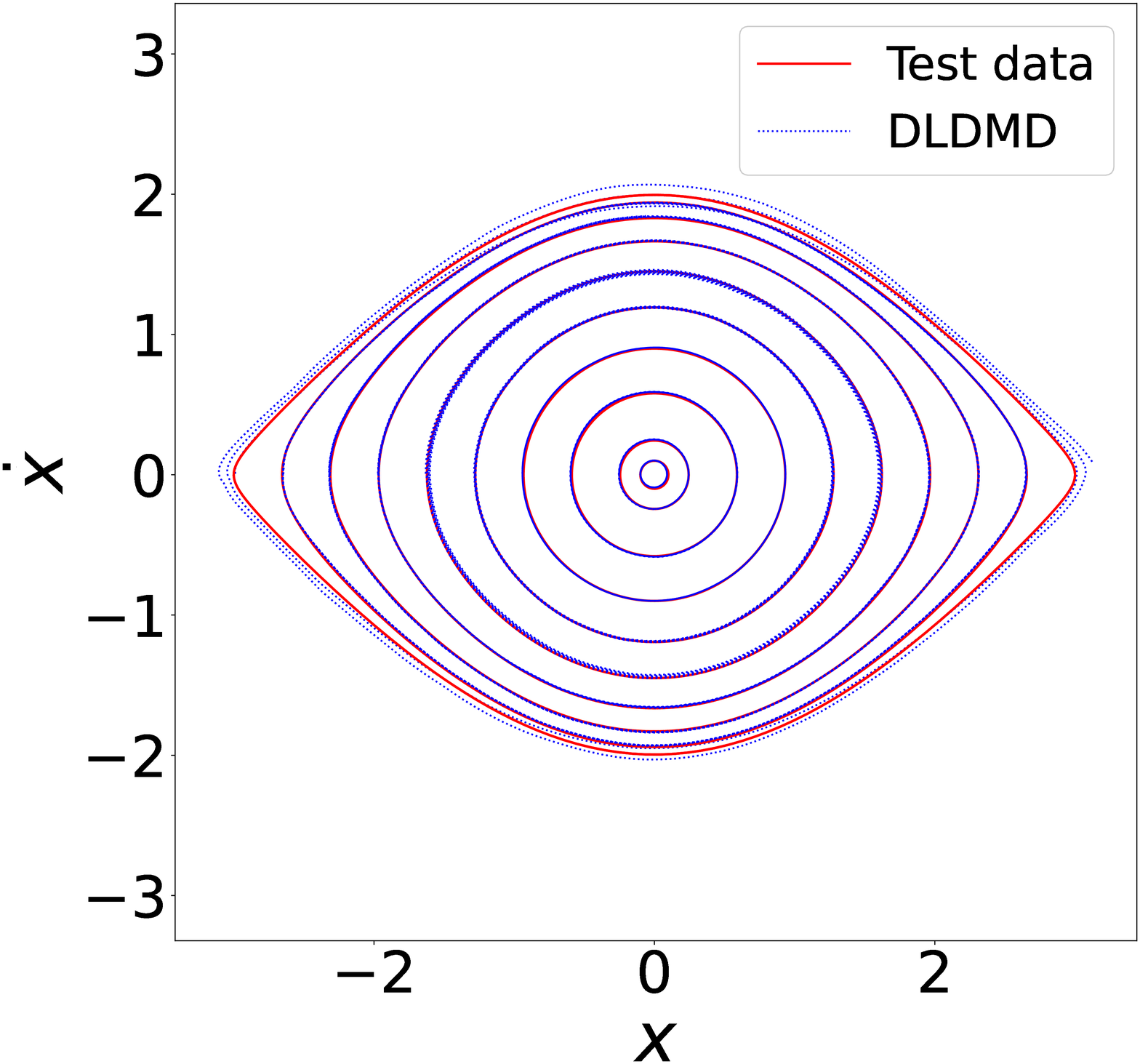}
    \caption{Results from the DLDMD as applied to a harmonic oscillator with test trajectories (solid lines) and predicted trajectories from the DLDMD (dotted lines). The MSE averaged over the 2000 test trajectories is $10^{-3.7}$.}
    \label{fig:pen_results}
\end{figure}

\begin{figure}
    \hspace*{-1.5cm}
    \centering
    \includegraphics[width=0.65\linewidth]{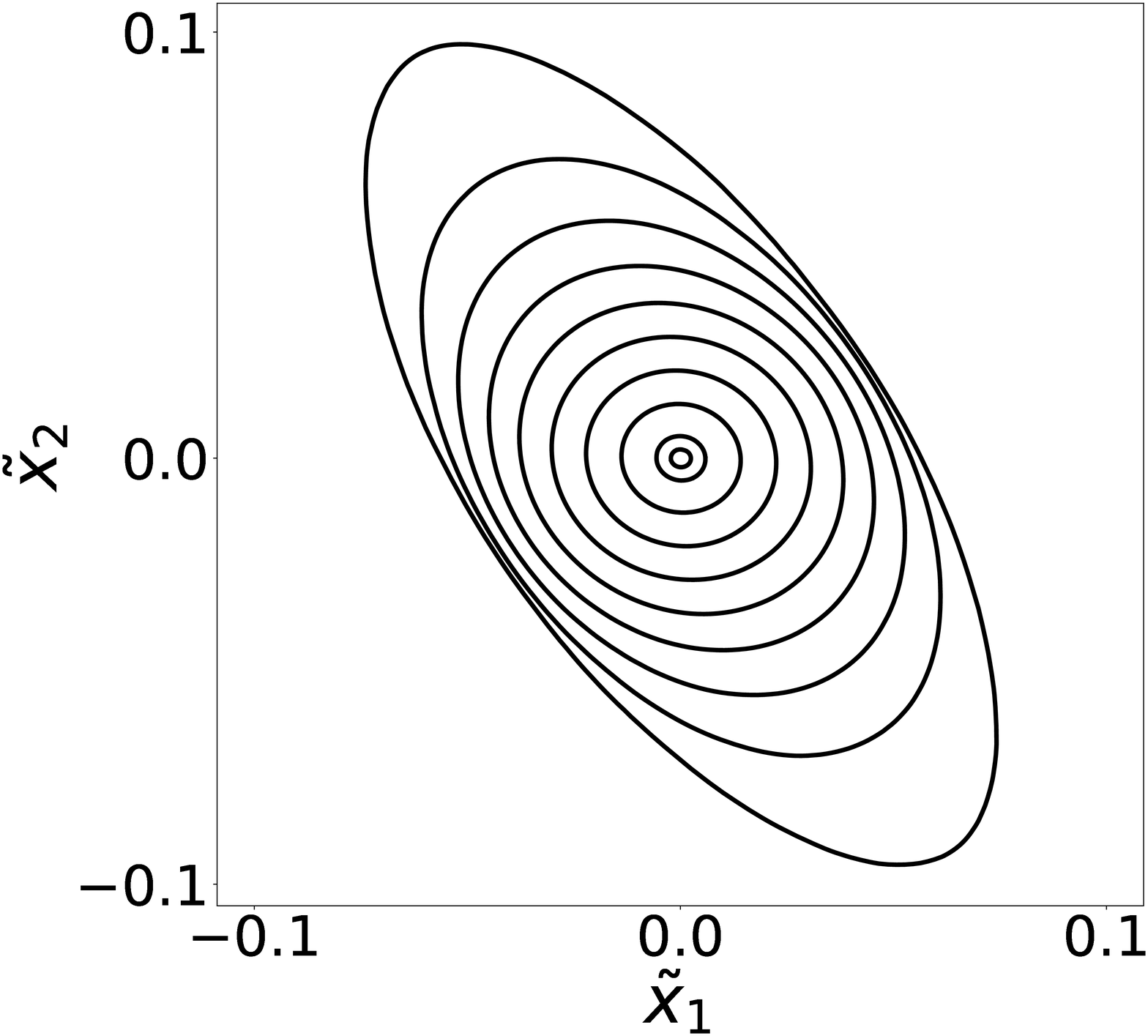}
    \caption{Test trajectories from the harmonic oscillator in the latent space coordinates. These are the trajectories on which the EDMD is applied.}
    \label{fig:pen_latent}
\end{figure}

Figure \ref{fig:pen_results} shows the DLDMD has found a mapping to and from coordinates in which it can apply the EDMD with fair precision and stability. This is achieved outside the linear regime corresponding to small angle displacements of the pendulum. These outer trajectories exhibit increasing nonlinearities; yet, nevertheless, the DLDMD is able to adapt to them with minimal assumptions in the model. For this example, we found that a latent dimension of $N_{o}=N_{s}=2$ produced the most parsimonious results. Taking the mean-squared error (MSE) for each trajectory and then averaging across all 2000 in the test set, we obtain a DLDMD loss of $10^{-3.7}$. 

Figure \ref{fig:pen_latent} plots the latent-space coordinates used in the EDMD step of the DLDMD. Here we see how the method has used the encoder network to morph the original test trajectories into a system that has less nonlinearity for the EDMD to overcome, in particular for the orbits near the separatrix. Indeed, by examining in Figure \ref{fig:pen_fft} the Fourier spectrum of each of the encoded coordinates, $(\tilde{x}_{1}, \tilde{x}_{2})$, we arrive at a fundamental innovation of the DLDMD method.  As we see, the embedded trajectory has a nearly monochromatic Fourier transform, showing that our neural network has learned embeddings and corresponding submersions which nonlinearly map the dynamics onto what would often be described in the dynamical systems literature as fundamental modes. Note, the trajectory illustrated in this figure corresponds to that in Figure \ref{fig:pen_results} which is the one that is closest to the separatrix and exhibits the most nonlinearity in the test set. We emphasize here that these latent-space coordinates are learned with no parameterization or assumptions from the user other than those in the loss function; see Equation \eqref{eqn:lossfun}. Moreover, we see in Figure \ref{fig:pen_eigs} plots of the eigenvalues for the 10 test trajectories from Figure \ref{fig:pen_results}. Each eigenvalue and its conjugate pair is essentially on the unit circle, so that this plot shows us how each embedded trajectory is governed by a single frequency of oscillation for all time. This in part echoes the relatively monochromatic Fourier spectra seen in Figure \ref{fig:pen_fft}.

\begin{figure}
    \hspace*{-1cm}
	\centering
	\begin{tabular}{c}
		\includegraphics[width=0.825\linewidth]{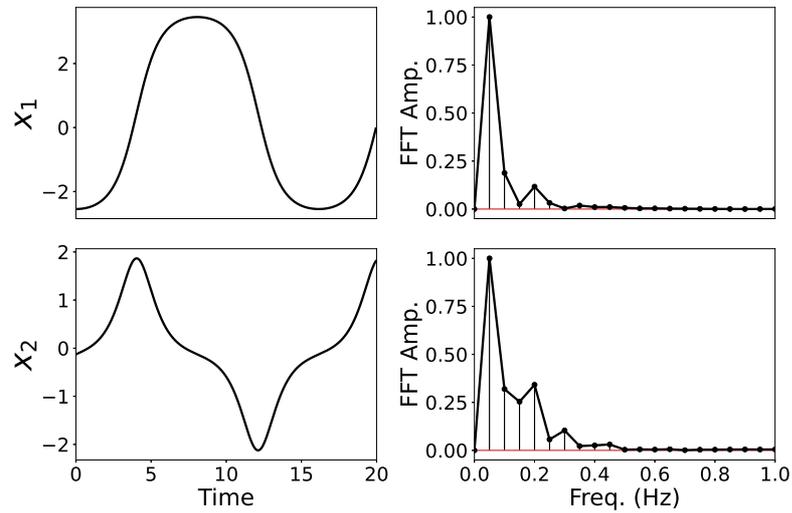}\\ (a) \\
		\includegraphics[width=0.825\linewidth]{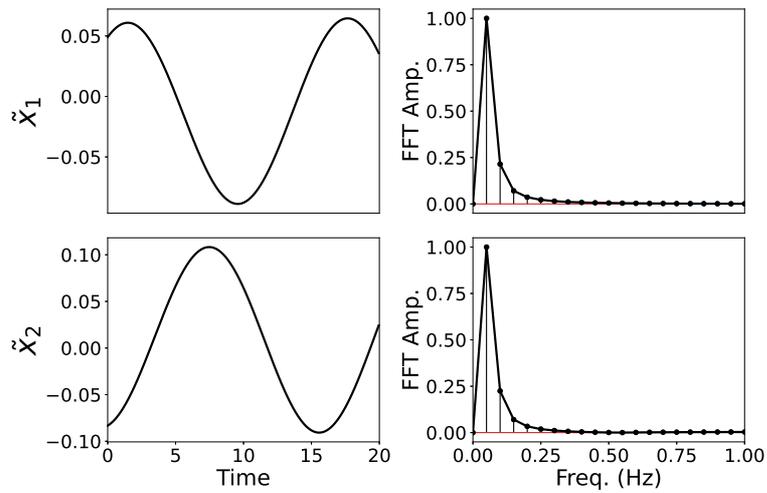}\\ (b)
	\end{tabular}
	\caption{Phase-space coordinates (a) and latent-space coordinates  (b) along with their affiliated normalized FFTs for the harmonic oscillator system. The test trajectory depicted in panel (a) corresponds to the outermost trajectory in Figure \ref{fig:pen_results} which lies nearest to the separatrix.}
	\label{fig:pen_fft}
\end{figure}

\begin{figure}
    \hspace*{-1.5cm}
	\centering
	\includegraphics[width=0.65\linewidth]{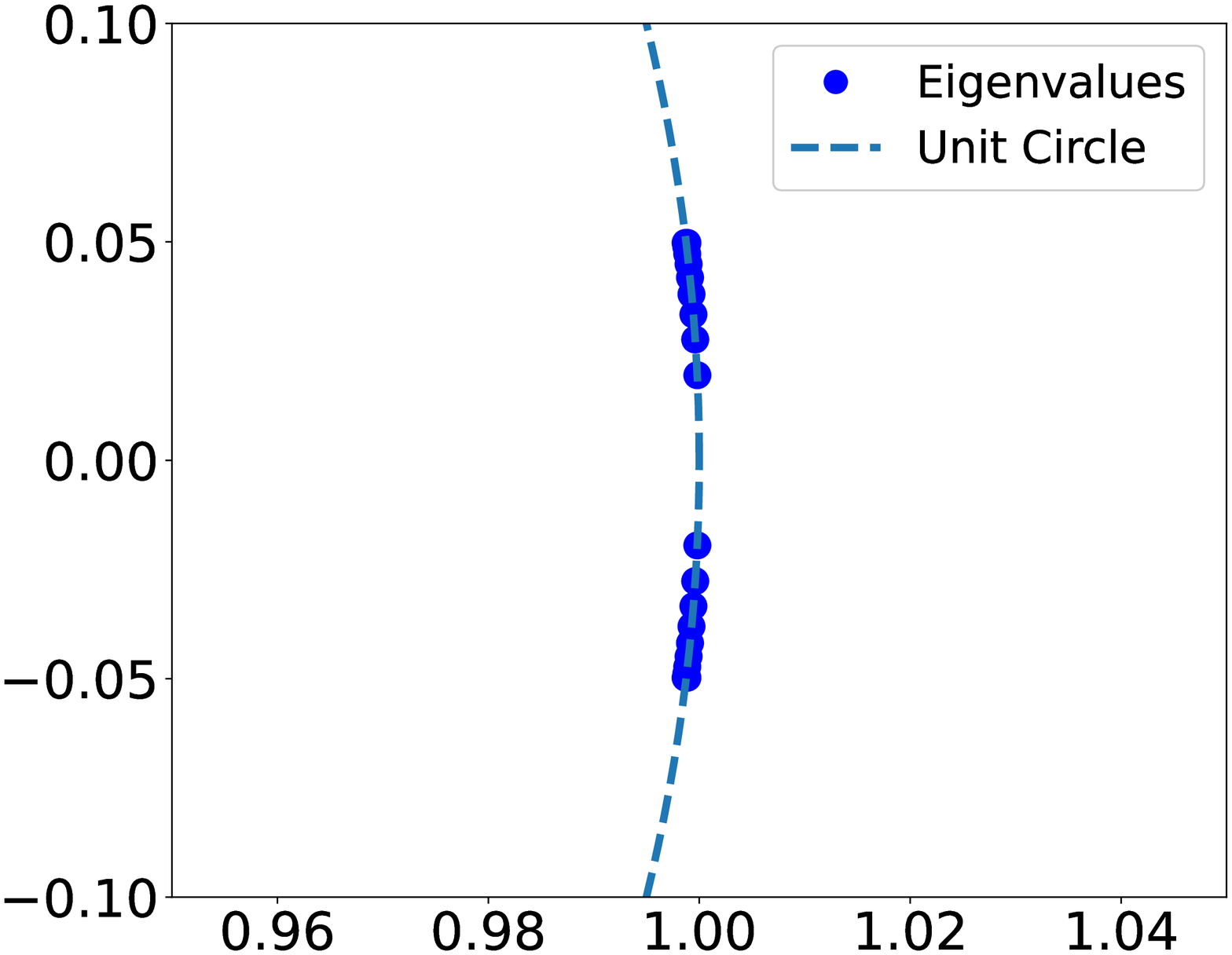}
	\caption{DLDMD eigenvalues for the 10 test trajectories in Figure \ref{fig:pen_results} from the harmonic oscillator.}
	\label{fig:pen_eigs}
\end{figure}

Of course, the harmonic oscillator only consists of closed orbits around a single center.  Therefore, it only contains one family of trajectories. As is shown in the next sections, when the dynamical system increases in complexity to include saddles, limit cycles, or chaos, we are still able to successfully generate global linearizations of the flow by increasing the embedding dimension $N_{o}$.

\subsection{The Duffing Equation: Two Centers}\label{sec:duffing}
The Duffing system is another weakly nonlinear oscillator with more complex behavior than the undamped pendulum. Without a driving force, the Duffing system is described by the double-well potential Hamiltonian system,
\begin{align*}
	\dot{x}_{1} &= x_{2}, \\[1ex]
	\dot{x}_{2} &= x_{1} - x_{1}^3.
\end{align*}
Here we are testing whether the DLDMD can cope with closed orbits that are not all oriented about a single center. Figure \ref{fig:duff_results} shows the reconstruction capability of the DLDMD over the unforced Duffing oscillator. For this system, we found that a latent dimension of $N_{o} = 3$ produced the best results.  Note, more on choosing the appropriate latent dimension is discussed in Section \ref{sec:conclusionsfuture}. Because $N_{o}=3$, we are still able to easily visualize the embedding, see Figure \ref{fig:duff_latent}, and find that the trajectories now follow nearly circular orbits on this higher-dimensional manifold. As we see, it appears that the homoclinic connections require an additional latent dimension in order to separate the three types of motion in the phase-space coordinates.  These three types are paths orbiting the left center, paths orbiting the right center, and paths orbiting outside the separatrix.

\begin{figure}
    \hspace*{-1.5cm}
	\centering
	\includegraphics[width=0.65\linewidth]{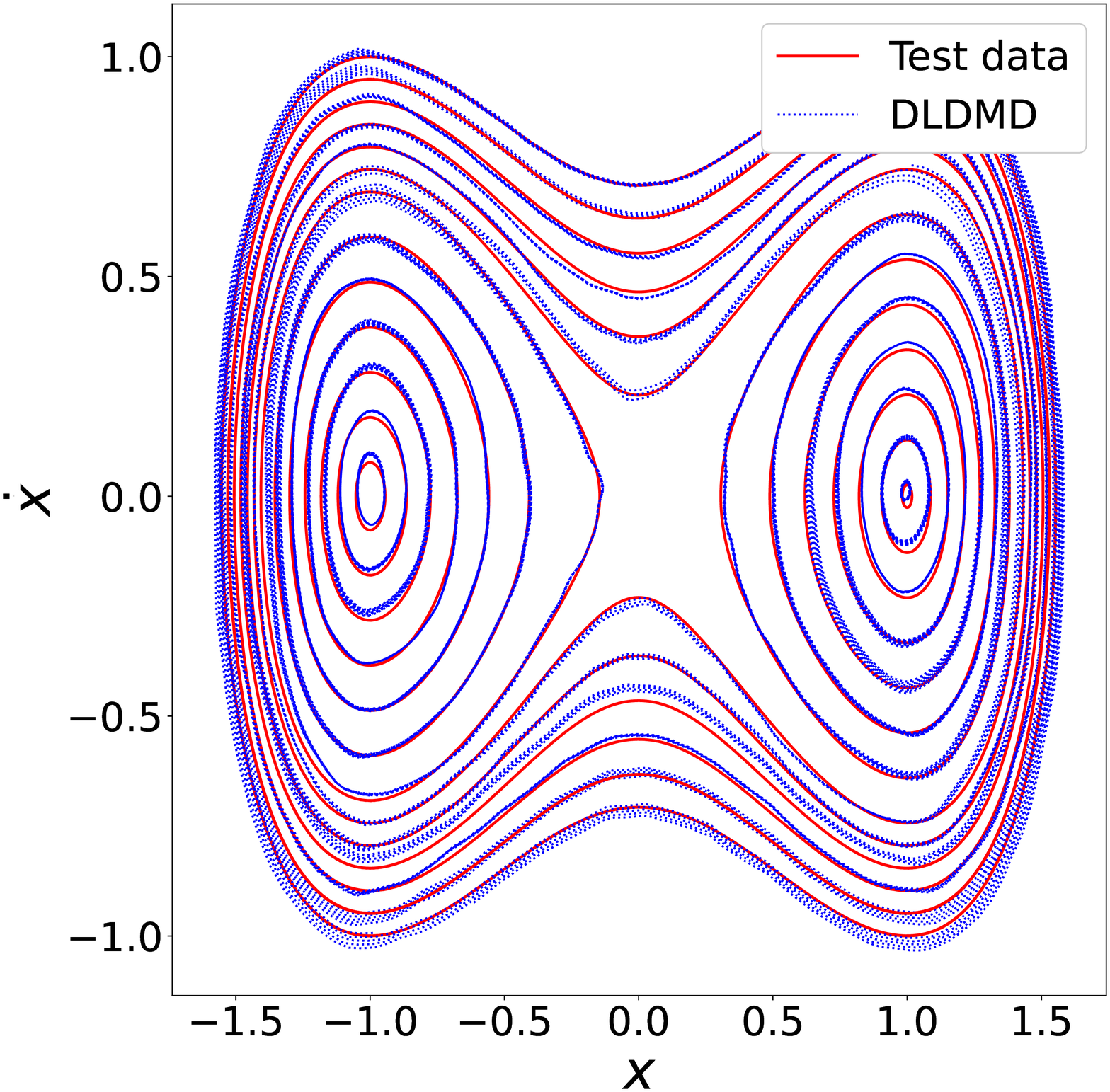}
	\caption{Results from the DLDMD as applied to the Duffing system with test trajectories (solid lines) versus predicted values from the DLDMD (dotted lines) in phase-space. The average MSE loss is $10^{-3.4}$.}
\label{fig:duff_results}
\end{figure}

\begin{figure}
    \hspace*{-1.5cm}
	\centering
	\includegraphics[width=0.65\linewidth]{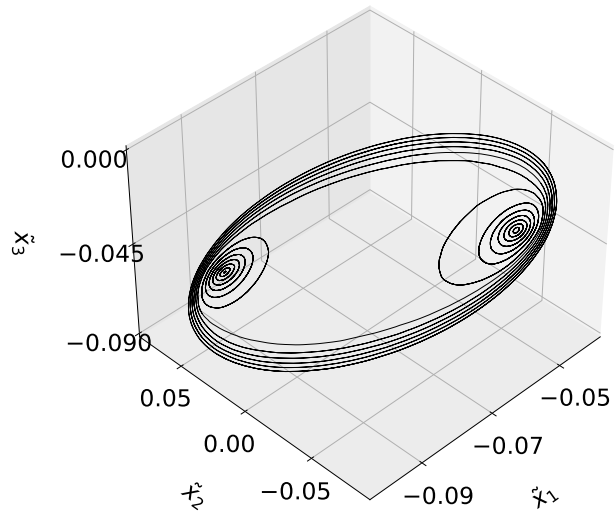}
	\caption{Test trajectories for the Duffing system in the latent-space coordinates.  These are the trajectories on which the EDMD is applied.}
\label{fig:duff_latent}
\end{figure}

As seen in Figure \ref{fig:duff_fft}, by simply adding one additional embedding dimension, we again find that the latent-space coordinates have nearly monochromatic Fourier spectra.  Furthermore, when we examine in Figure \ref{fig:duff_eigs} the affiliated eigenvalues for several of the test trajectories, we see that while each orbit is characterized by a unit-magnitude complex-conjugate pair of eigenvalues, as well as an eigenvalue exactly equal to one.  This strictly real eigenvalue corresponds to the EDMD accurately computing the temporal average of the time series, which for the Duffing system is determined by which fixed point an orbit oscillates around.  Thus, the higher embedding dimension here allows the EDMD to accurately account for this difference between trajectories.  

\begin{figure}
	\centering
	\begin{tabular}{c}
        \hspace*{-1cm}
		\includegraphics[width=0.825\linewidth]{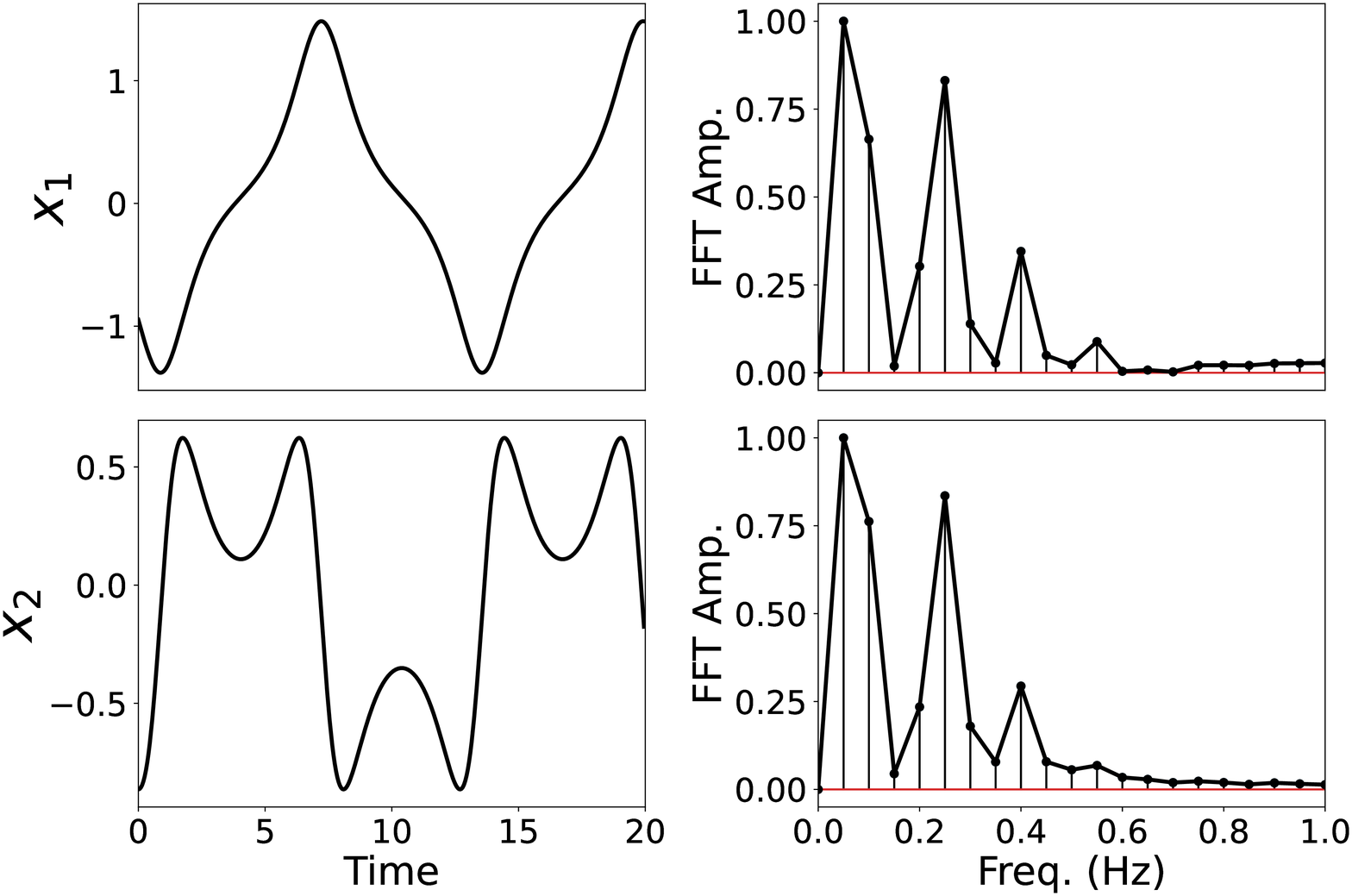}\\ (a) \\
        \hspace*{-1cm}
		\includegraphics[width=0.825\linewidth]{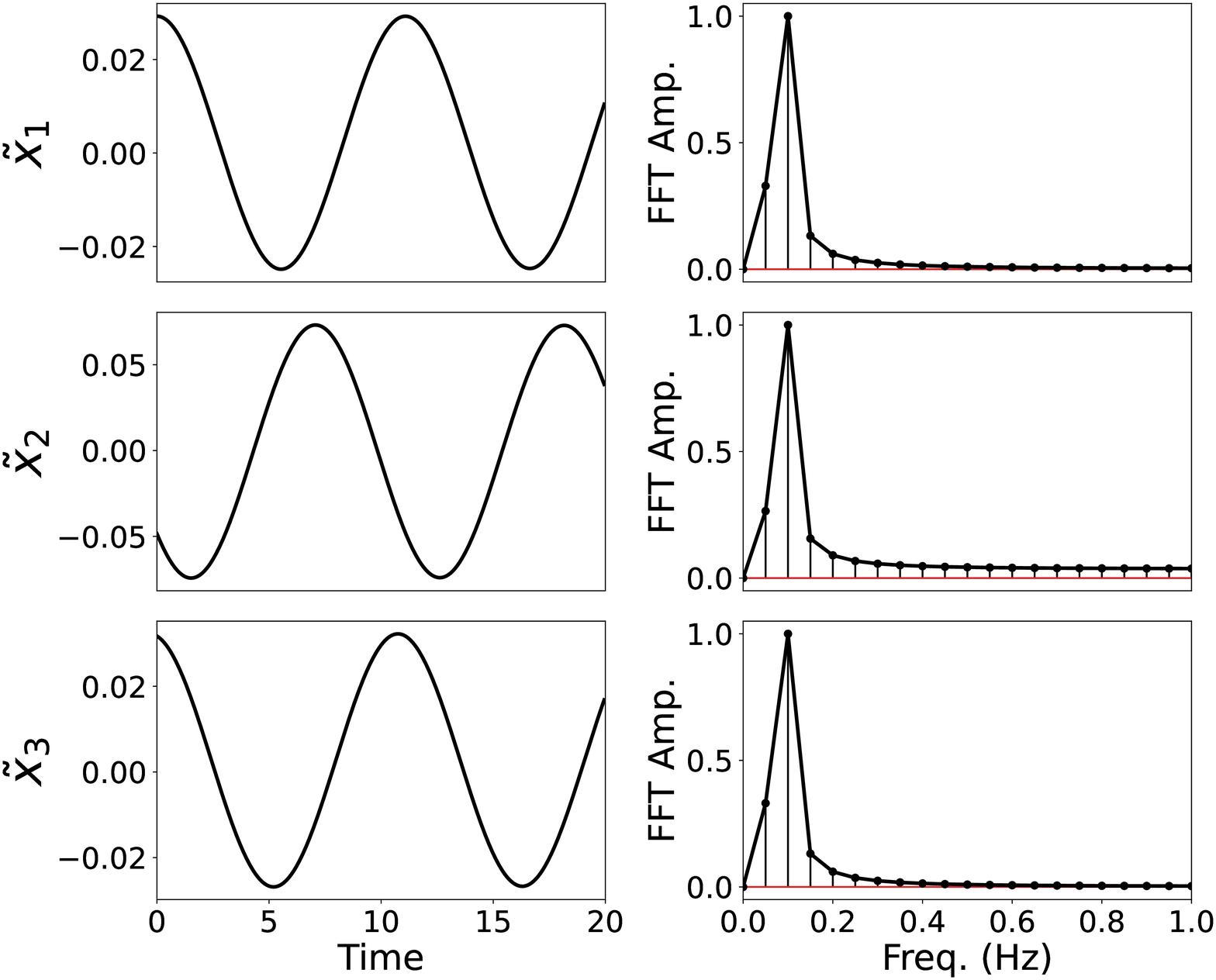}\\ (b)
	\end{tabular}
	\caption{Phase-space coordinates (a) and latent-space coordinates  (b) along with their affiliated normalized FFTs for the Duffing system. The test trajectory depicted in panel (a) corresponds to the innermost trajectory in Figure \ref{fig:duff_results} that encompasses both centers.}
	\label{fig:duff_fft}
\end{figure}

\begin{figure}
    \hspace*{-1cm}
	\centering
	\includegraphics[width=0.65\linewidth]{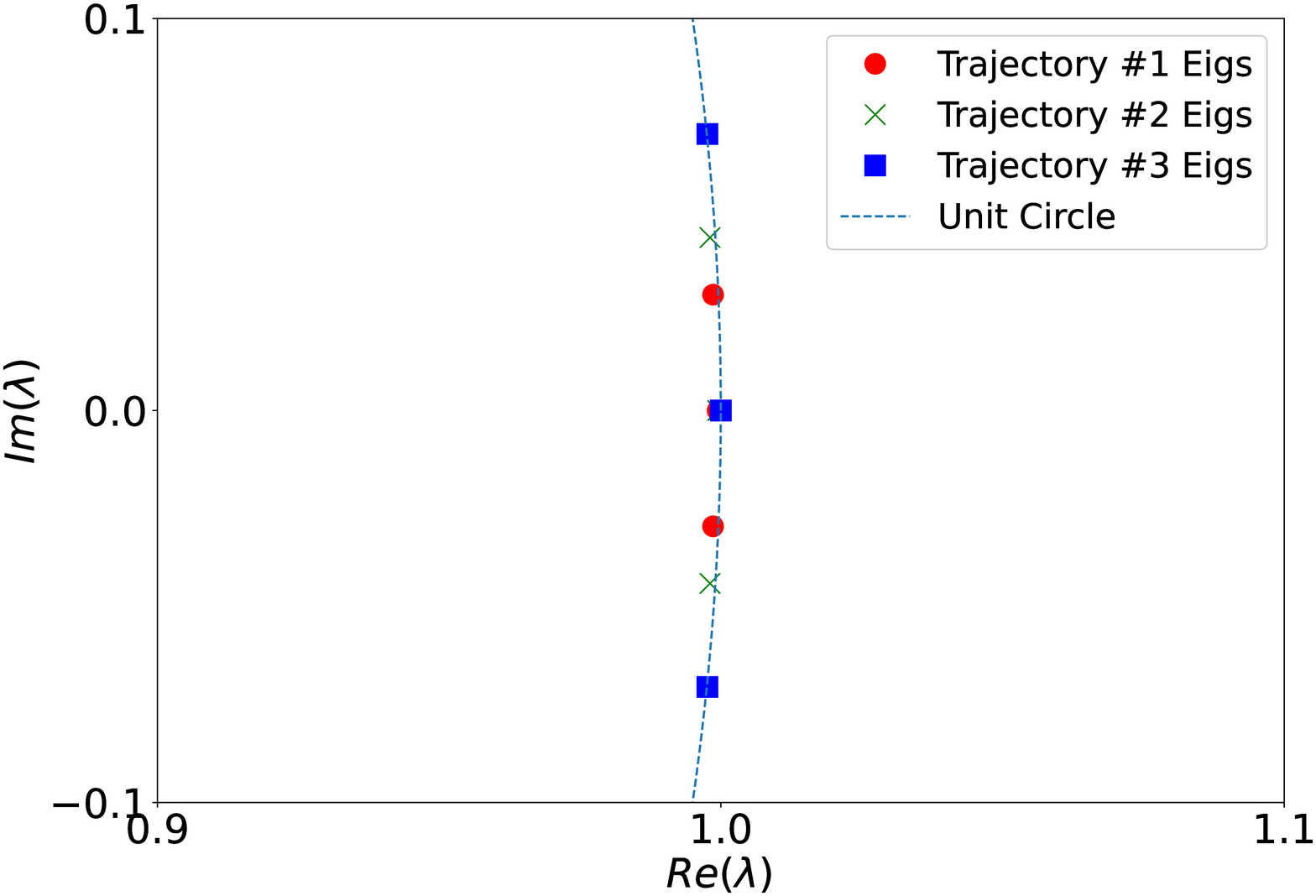}
	\caption{DLDMD eigenvalues for 3 trajectories for the Duffing system corresponding to an outer orbit (red dot), a left-center orbit (green x), and a right-center orbit (blue square). Note that all three types of orbits have an eigenvalue $\lambda=1$ corresponding to the average of each orbit.}
	\label{fig:duff_eigs}
\end{figure}

\subsection{The Van der Pol Oscillator: Attraction to a Slow/Fast Limit Cycle}\label{sec:vdp}
The Van der Pol oscillator, described by the parameterized dynamical system
\begin{align*}
	\dot{x}_{1} &= x_{2}, \\[1ex]
	\dot{x}_{2} &= \mu (1 - x_{1}^2) x_{2} - x_{1},
\end{align*}
has for positive values of $\mu$ a globally attractive limit cycle. All presented results use $\mu=1.5$; however, the DLDMD has been tested out to $\mu=4$ with no modifications to the algorithm or hyperparameters. The limit cycle itself is made up of slow and fast submanifolds thereby producing multiscale behavior. This system pushes the DLDMD much further than the harmonic oscillator and Duffing systems, for it must now account for the attraction onto the limit cycle as well as the multiscale periodic motion. Furthermore, the attraction onto the limit cycle involves a transient growth or decay term for initial conditions starting inside or outside the limit cycle respectively. Because of this complexity, we found that the DLDMD performed well with $N_{o}=8$, though for this system the choice was not as stringent as for the previous two cases. Our choice of $N_{o}=8$ was due to its performance after 1,000 epochs, which was the maximum used for all of the dynamical systems studied in this paper.  However, a reasonable reconstruction and prediction error could likely be obtained with slightly fewer or slightly more embedding dimensions given enough training epochs. 

\begin{figure}
    \hspace*{-1.5cm}
	\centering
	\includegraphics[width=0.65\linewidth]{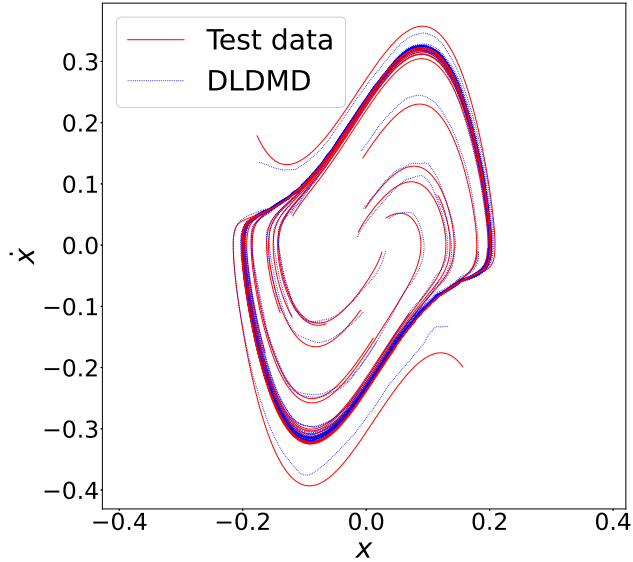}
	\caption{Results from the DLDMD as applied to the Van der Pol oscillator with $\mu=1.5$.  Test trajectories (solid lines) and the predicted values from the DLDMD (dotted lines) show paths in phase space. The average MSE loss is $10^{-2.8}$.}
\label{fig:vdp_results}
\end{figure}

Figure \ref{fig:vdp_results} shows the DLDMD reconstruction of the Van der Pol system for several test trajectories. The MSE loss averaged over all 2000 test trajectories is $10^{-2.8}$. Figure \ref{fig:vdp_fft} again illustrates how an encoded trajectory is transformed from one with a relatively large spread in its Fourier spectrum to a set of coordinates whose spread in Fourier space is much reduced, and Figure \ref{fig:vdp_eigs} plots the corresponding eigenvalues.  However, unlike the previous two cases, we do see some slight deviations away from strictly periodic motion, reflecting the transients in the underlying dynamics.  These transient phenomena are also reflected by five of the eigenvalues being off the unit circle, indicating that dynamics along the affiliated coordinates decay in time leaving only the oscillatory modes.    

\begin{figure}
	\centering
	\begin{tabular}{c}
        \hspace*{-1cm}
		\includegraphics[width=0.65\linewidth]{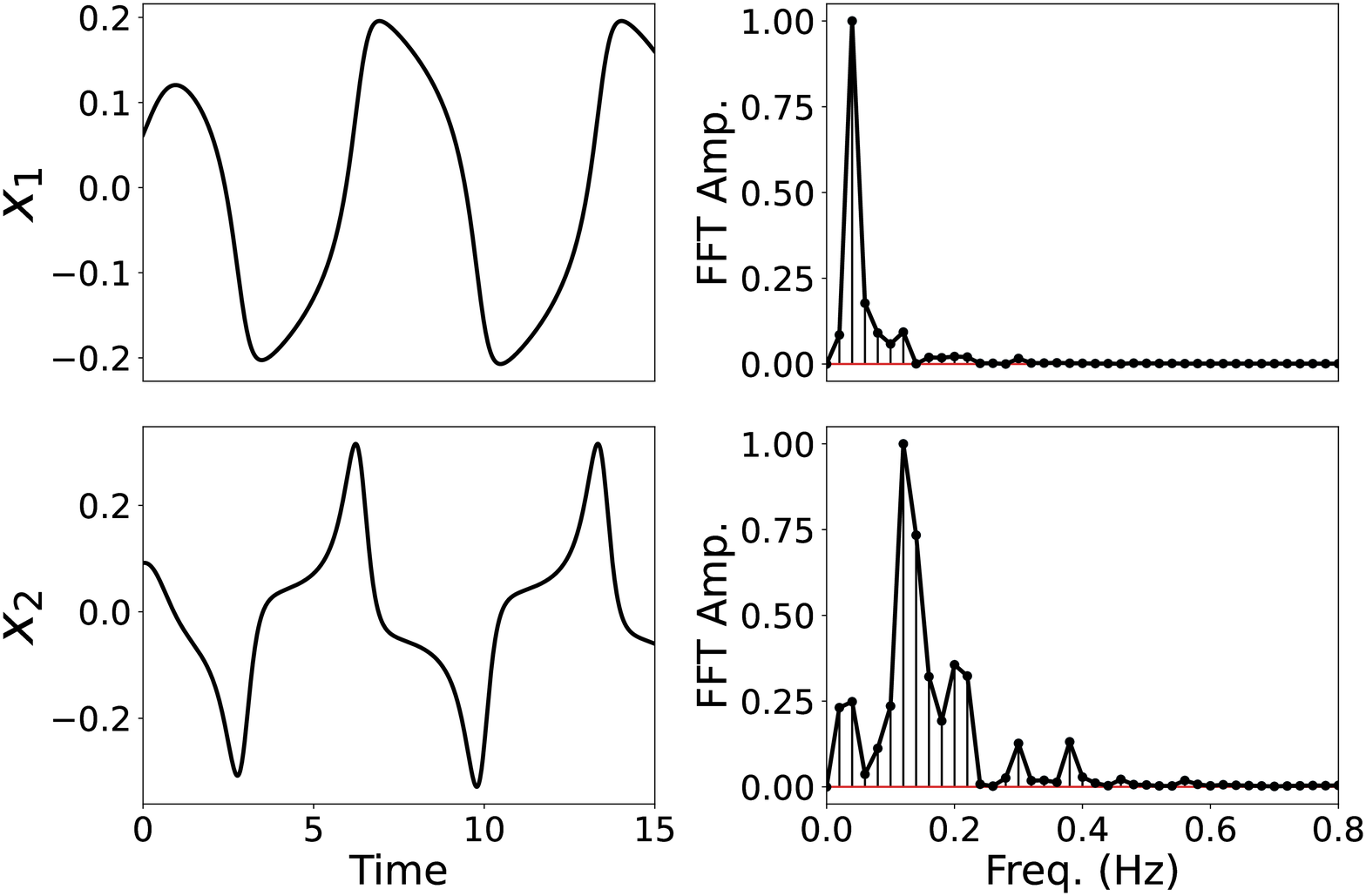}\\ (a) \\
        \hspace*{-1cm}
		\includegraphics[width=0.75\linewidth]{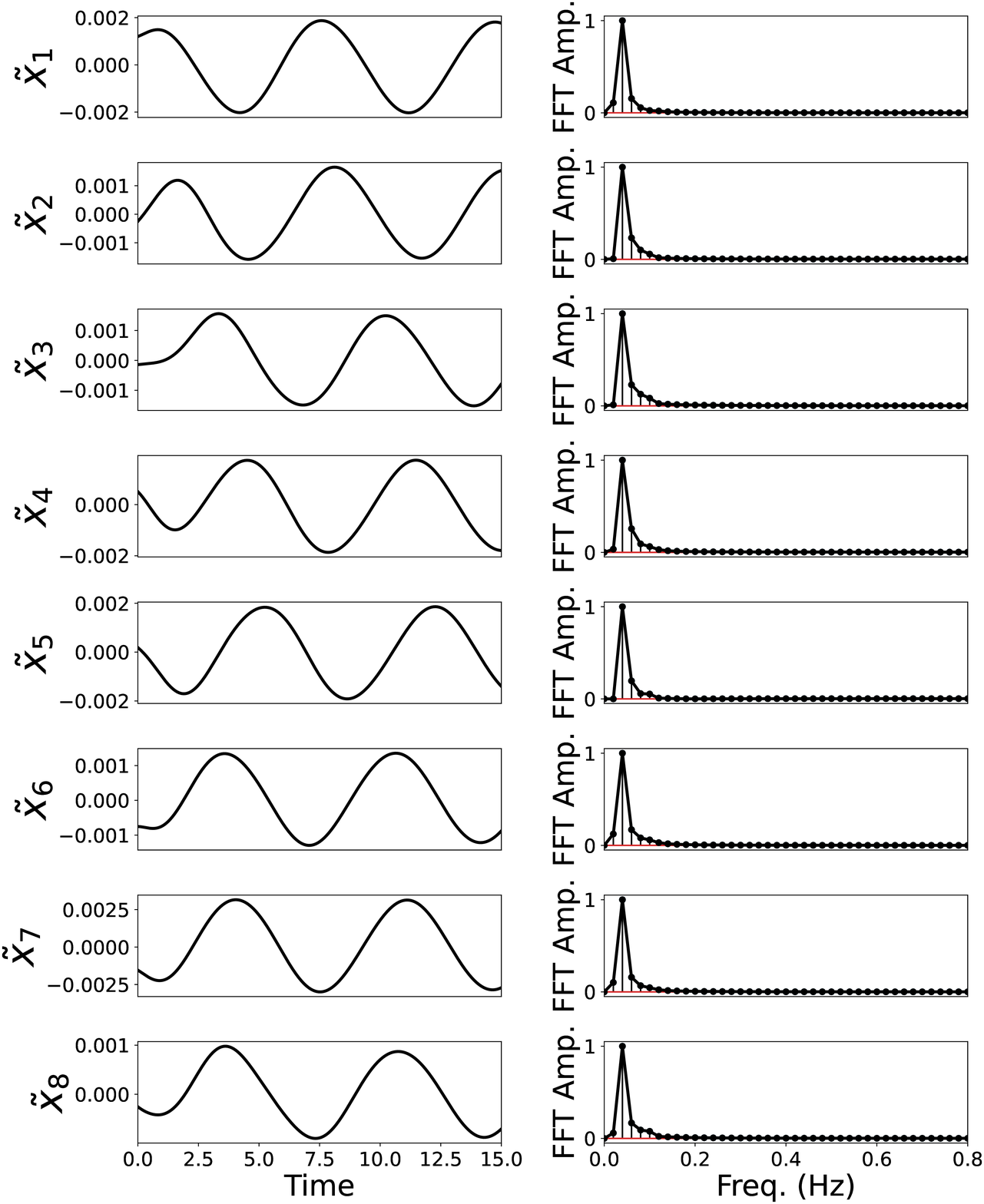}\\ (b)
	\end{tabular}
	\caption{Phase-space coordinates (a) and latent-space coordinates  (b) along with their affiliated normalized FFTs for the Van der Pol system. The test trajectory depicted in panel (a) corresponds to one of the trajectories in Figure \ref{fig:vdp_results} that begins near the origin and grows outward onto the limit cycle.}
	\label{fig:vdp_fft}
\end{figure}

\begin{figure}
    \hspace*{-1.5cm}
	\centering
	\includegraphics[width=0.55\linewidth]{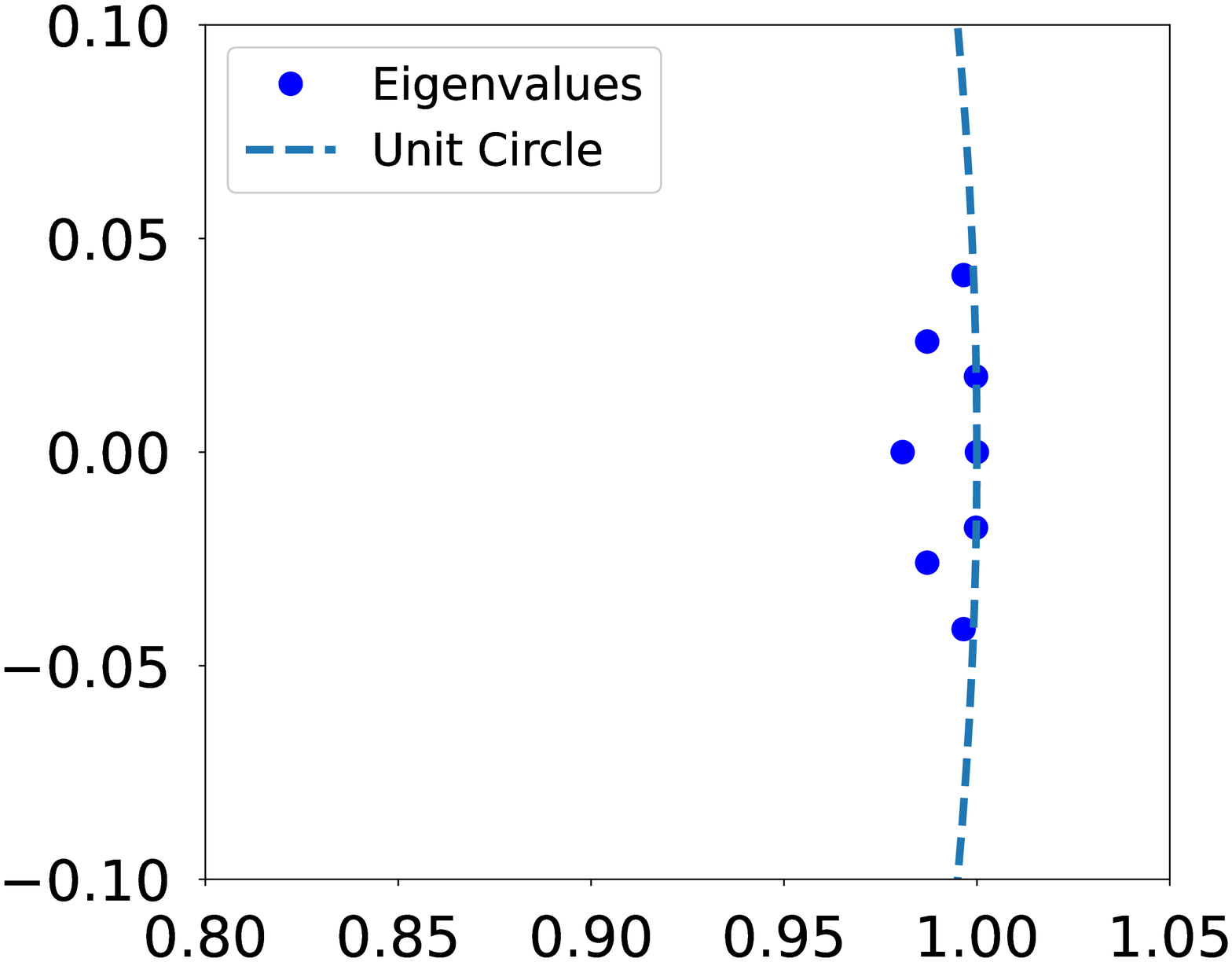}
	\caption{DLDMD eigenvalues for the Van der Pol trajectory corresponding to Figure \ref{fig:vdp_fft}.  Those eigenvalues inside the unit circle indicate transient phenomena, reflecting the transient behavior in the underlying dynamics.}
	\label{fig:vdp_eigs}
\end{figure}

\subsection{The Lorenz-63 system: Chaos}\label{sec:lorenz}
The Lorenz-63 system, described by the equations
\begin{align*}
	\dot{x}_{1} &= \sigma (x_{2} - x_{1}), \\[1ex]
	\dot{x}_{2} &= x_{1} (\rho - x_{3}) - x_{2} \\[1ex]
	\dot{x}_{3} &= x_{1}x_{2} - \beta x_{3},
\end{align*}
with parameters $\sigma=10$, $\rho=28$, and $\beta=8/3$ generates chaotic trajectories with a strange attractor. This system provides categorically different dynamics than the previous three examples, but the DLDMD is able to discover the attractor structure even though its overall pointwise prediction is poor; see Figure \ref{fig:lorenz_recon}. This is seen more readily in Figure \ref{fig:lorenz_preds} by examining each component of the test versus predicted trajectory.  We see, in particular, that the DLDMD predicted trajectory exhibits a lobe switching pattern that is close to that of the test trajectory. This is an especially pleasing result given that the model was trained on trajectories of length $t_{f}=3$, while the test trajectory shown here was extended to twice that to $t_{f}=6$. The latent dimension used to get these results was $N_{\text{o}}=4$. This is somewhat surprising in that a system with chaotic trajectories only required one additional dimension in order to obtain decent global linear representations via EDMD whereas the limit cycle of Van der Pol required upward of six more embedding dimensions. The choice of $N_{o}$ was quite inflexible compared to that used for the Van der Pol equation, with the total loss during training increasing by several orders of magnitude for $N_{o}\ge 5$. 

Moreover, as seen by comparing Figures \ref{fig:lorenz_fft} (a) and (b), the DLDMD finds a set of latent-space coordinates for the Lorenz-63 system that, while no longer monochromatic due to clearly visible two-frequency or {\it beating} phenomena, are far more sparsely represented in their affiliated Fourier spectrum than the original phase-space coordinates.  That said, the $\tilde{x}_{3}$ coordinate would seem to track some aperiodic behavior in the dynamics, though longer simulation times would be necessary to determine what exactly is being represented via this transformation. Overall though, these plots further reinforce the result that the DLDMD can generally find embeddings in which the Fourier spectral representation of a given trajectory is far more sparse.  Likewise, as seen in Figure \ref{fig:lorenz_eigs}, we see the DLDMD again finds eigenvalues either on or nearly on the unit circle, reflecting the largely oscillatory behavior of the latent-space trajectories.  For those eigenvalues just inside the unit circle, the implied weak transient behavior could be more of an artifact of limited simulation time.  Resolving this issue is a subject of future work.  

\begin{figure}
    \hspace*{-1cm}
	\centering
	\includegraphics[width=1\linewidth]{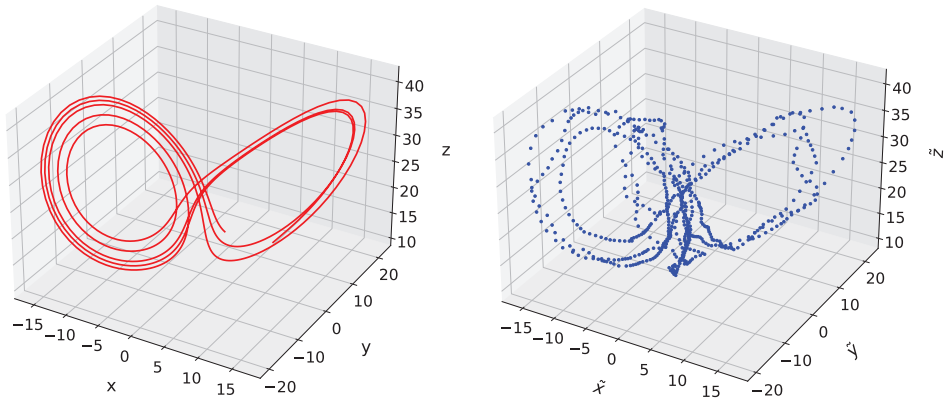}
	\caption{Test trajectory (solid) and DLDMD prediction (dotted) on the Lorenz-63 system. The pointwise MSE loss of the trajectory shown is $1.79$, so while poor in point-to point prediction, the DLDMD is able to approximate the strange attractor structure in the original phase-space coordinate system.}
	\label{fig:lorenz_recon}
\end{figure}

\begin{figure}
    \hspace*{-1cm}
	\centering
	\includegraphics[width=0.85\linewidth]{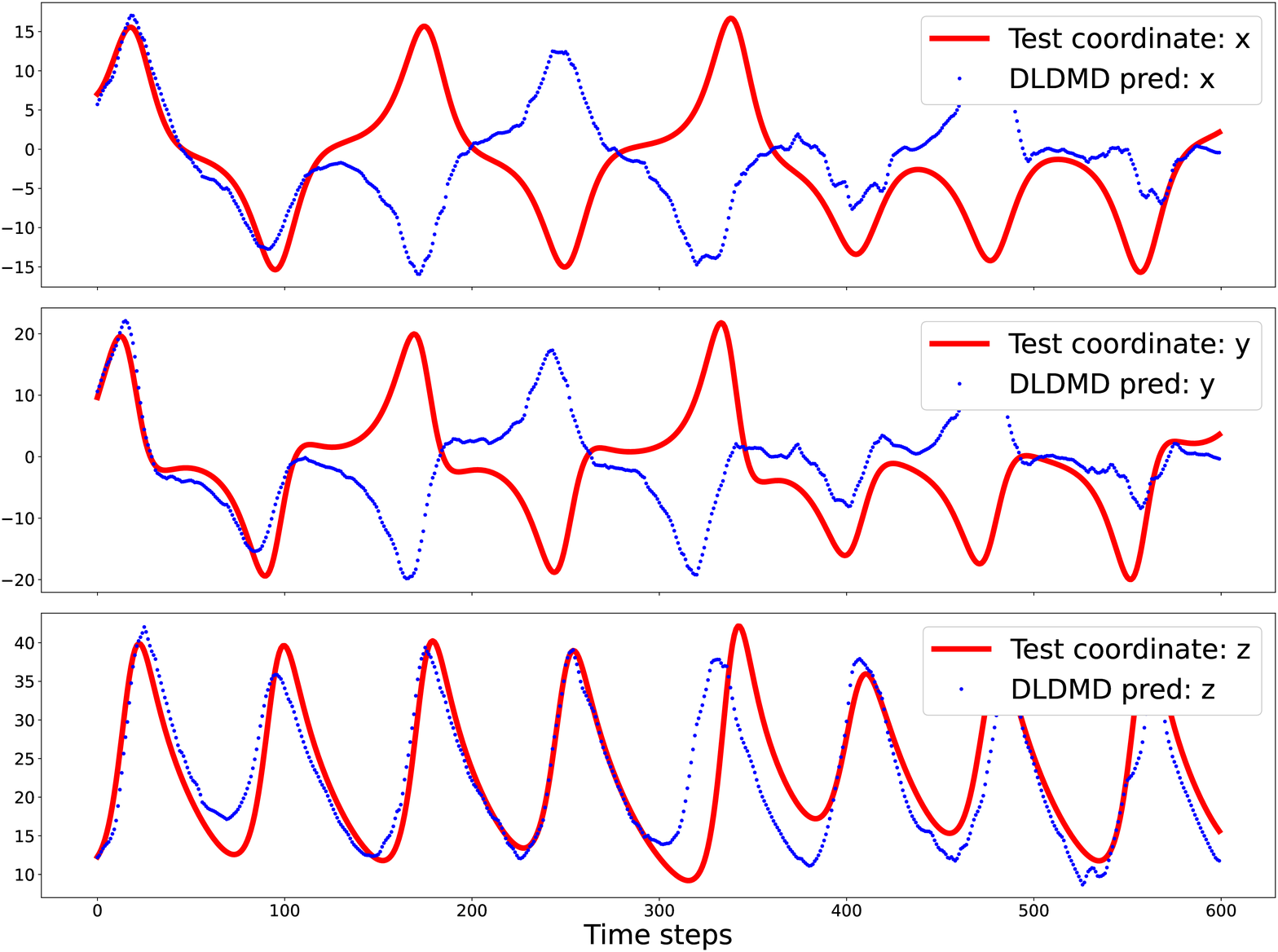}
	\caption{Each component of the test trajectory (solid) and DLDMD prediction (dotted) on the Lorenz-63 system. The DLDMD model was trained on trajectories with 300 time steps, while the predicted trajectories shown here are taken 600 time steps forward.}
	\label{fig:lorenz_preds}
\end{figure}

\begin{figure}
	\centering
	\begin{tabular}{c}
		\includegraphics[width=0.75\linewidth]{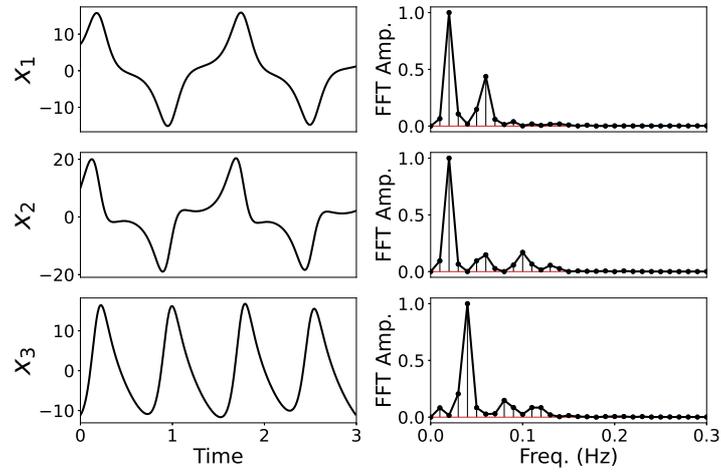}\\ (a) \\
		\includegraphics[width=0.75\linewidth]{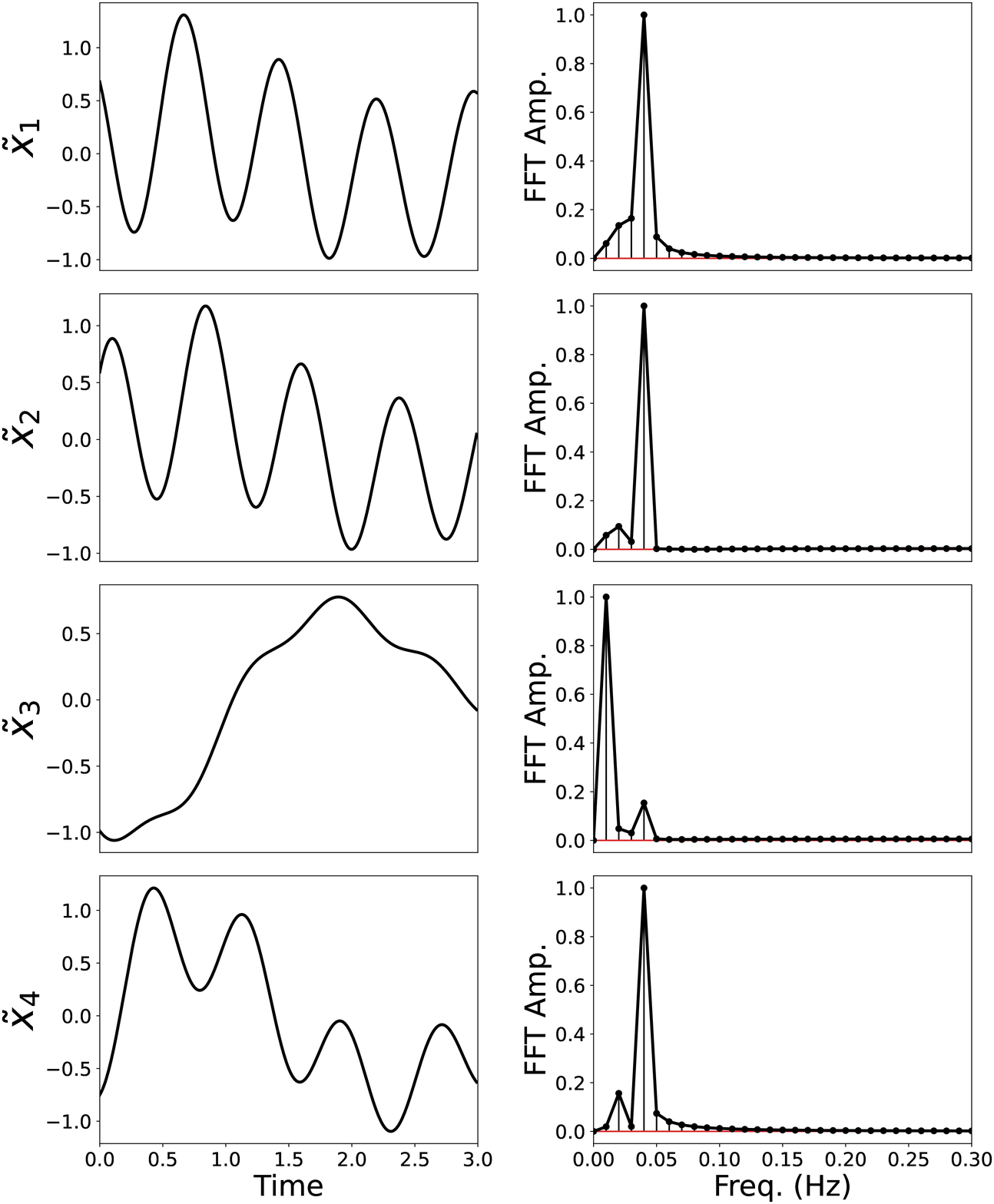}\\ (b)
	\end{tabular}
	\caption{Phase-space coordinates (a) and latent-space coordinates  (b) along with their affiliated normalized FFTs for the Lorenz 63 system. The test trajectory depicted in panel (a) corresponds to the trajectory in Figure \ref{fig:lorenz_recon}.}
	\label{fig:lorenz_fft}
\end{figure}

\begin{figure}
    \hspace*{-1cm}
	\centering
	\includegraphics[width=0.55\linewidth]{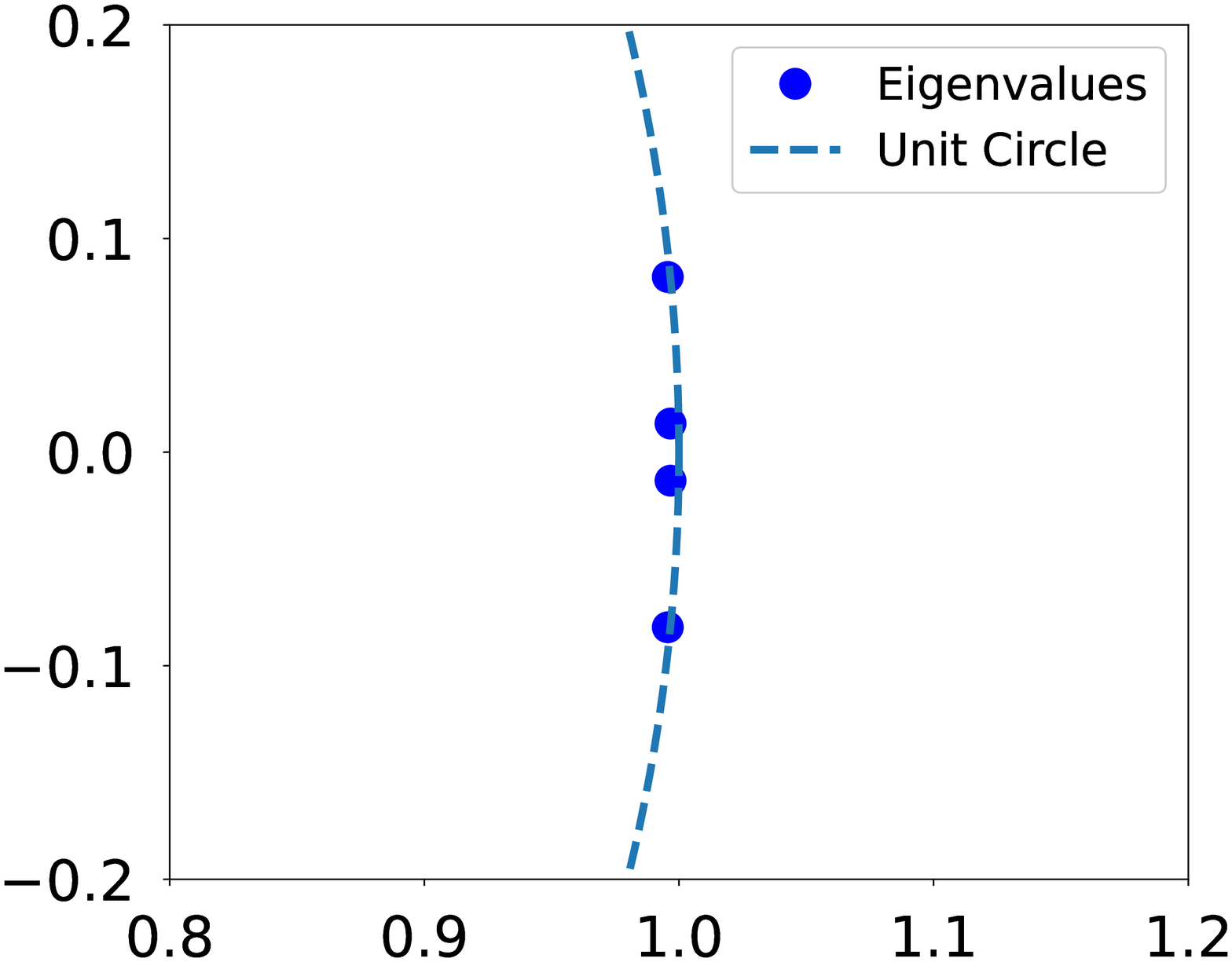}
	\caption{DLDMD eigenvalues for the Lorenz-63 trajectory corresponding to Figure \ref{fig:lorenz_fft}.  While mostly oscillatory, there is indication of very slowly decaying transient phenomena due to the two eigenvalues just inside the unit circle.}
	\label{fig:lorenz_eigs}
\end{figure}

\section{\label{sec:conclusionsfuture}Conclusions and Future Directions}
In this paper, we have developed a deep learning extension to the EDMD, dubbed DLDMD, which ensures, through the inclusion in our loss function of the terms $\mathcal{L}_{\text{dmd}}$ and $\mathcal{L}_{\text{pred}}$, both the local one step accuracy and global stability, respectively, in the generated approximations. This keeps us numerically close to satisfying the requirements stated in Ansatz \eqref{ansz1}, which is necessary for the success of the EDMD in the latent variables.  Likewise, by constructing a loss function to train the autoencoder using the diagram in Figure \ref{fig:comm_diag}, the DLDMD learns mappings to and from the latent space which are one-to-one. Thus we ensure that all Koopman modes and eigenvalues are captured in this latent space as well.  These results taken together ensure that the DLDMD finds a global linearization of the flow, facilitating a straightforward spectral characterization of any analyzed trajectory and an accurate prediction model for future dynamics.

\begin{table}[]
\centering
\begin{tabular}{|c|c|c|c|c|}
\hline
\textbf{System} & \textbf{Avg. Loss (MSE)} & \ ${\bf N_{o}}$ \ & \textbf{\# Params} & \textbf{sec./epoch} \\ \hline
Pendulum    & $3.69 \times 10^{-3}$ & 2 & 67K & 6  \\ \hline
Duffing     & $3.47 \times 10^{-3}$ & 3 & 67K & 8  \\ \hline
Van der Pol & $2.87 \times 10^{-2}$ & 8 & 68K & 14 \\ \hline
Lorenz 63   & $1.79 \times 10^{0}$  & 3 & 68K & 7  \\ \hline
\end{tabular}
\caption{Tabulated results for the DLDMD for each dynamical system presented. The loss for each is computed as the MSE of the reconstructed trajectory and averaged over all 2,000 paths in the test data set. The size of the embedding dimension, number of trainable parameters, and training wall-time in seconds per epoch is provided for each.}
\label{tab:table_results}
\end{table}

Moreover, we have developed a ML approach which requires a relatively minimal number of assumptions to ensure successful training. We have demonstrated this across a relatively wide array of dynamical phenomena including chaotic dynamics; see Table \ref{tab:table_results} for a performance summary of the DLDMD across all presented examples. We are able to cover so much ground by way of implementing relatively straightforward higher-dimensional embeddings via the encoder network, $\mathcal{E}$.  The latent-space coordinates of each trajectory generally show that the encoder produces modes with near monochromatic Fourier spectra; again see Figures \ref{fig:pen_fft}, \ref{fig:duff_fft}, \ref{fig:vdp_fft}, and \ref{fig:lorenz_fft}. This is an especially compelling result of the DLDMD, and it speaks to the power of deep learning methods that such elegant representations can be discovered with so little user guidance.  However, by lifting into higher-dimensional latent spaces we do lose the strict topological equivalence in \cite{bollt1} and cannot guarantee that all spectra are invariant. This issue does not seem to manifest in any measurable way in our results, but it should be kept in mind when pursuing future work.  

The DLDMD is not without other drawbacks. In particular, the latent dimension parameter $N_{o}$ is critical to the success or failure of the DLDMD. Unfortunately, there is no readily apparent method for choosing this embedding dimension before training. Therefore, the optimal $N_{o}$ had to be determined by simply training the model at successively larger values for $N_{o}$ and stopping once the error became too large or grew unstable. This approach is time consuming and an obvious disadvantage to the method. Determining more optimal ways to approach the latent dimensionality will be the subject of future research. 

Another drawback of the DLDMD is the number of trajectories used during training. For each example we generated 10,000 initial conditions sampled uniformly over the phase space of interest for training. Rarely do real-world datasets provide such a uniform sampling of the space. Methods to cope with sparse observations could potentially add far more utility to the method. Very high-dimensional systems that exhibit low-rank structure also present problems for DLDMD and the more conventional use of autoencoder networks for compression rather than lifting \cite{lusch2, azencot} may be more applicable. The chaotic trajectories of the Lorenz-63 system were clearly the most challenging for DLDMD to reproduce. This is evident in the low average MSE loss and in the relatively poor agreement between the test and predicted trajectory. For this reason, chaotic systems will likely require additional innovations, such as incorporating delay embeddings into the EDMD step \cite{arbabi1} and is the topic of future work.

Lastly, systems with noisy observations will pose significant challenges to DLDMD as it is reliant on the one-step prediction to enforce consistency in the latent spectra. In this regard, the method would seem to benefit from some of the alternative DMD approaches as found in \cite{brunton1}, an issue we will also pursue in future work.

\section*{Data Availability Statement}

The data that support the findings of this study are openly available at\\ https://github.com/JayLago/DLDMD.git

\section*{\label{sec:acknowledgements}Acknowledgements}
This work was supported in part by the Naval Information Warfare Center Pacific (NIWC Pacific) and the Office of Naval Research (ONR).  

The authors would like to thank Dr. Erik M. Bollt from the Department of Electrical and Computer Engineering at Clarkson University for his feedback and insights, as well as Joseph Diaz and Robert Simpson for their insightful comments.

\bibliography{ms}
\bibliographystyle{unsrt}
\end{document}